\documentclass[11pt]{article}
\usepackage{microtype}
\usepackage{graphicx}
\usepackage{caption}
\usepackage{subcaption}
\usepackage{comment}
\usepackage{booktabs} \usepackage{hyperref}

\usepackage{natbib}

\usepackage[utf8]{inputenc}
\usepackage{amsmath}
\usepackage{amsfonts}
\usepackage{amsthm}
\usepackage{amssymb}
\usepackage[algoruled,vlined,nofillcomment]{algorithm2e}
\usepackage{fullpage}

\usepackage{amsmath}
\usepackage{amsfonts}
\usepackage{amsthm}
\usepackage{graphicx}
\usepackage{url}
\usepackage{mathbbol}
\usepackage{amssymb}

\usepackage[utf8]{inputenc} \usepackage[T1]{fontenc}    \usepackage{hyperref}       \usepackage{url}            \usepackage{booktabs}       \usepackage{amsfonts}       \usepackage{nicefrac}       \usepackage{microtype}      \usepackage{multirow}
\usepackage{tabularx}
\usepackage{booktabs}
\usepackage{caption}
\usepackage{subcaption}

\DeclareSymbolFontAlphabet{\amsmathbb}{AMSb}\newcommand{\R}{\amsmathbb{R}}
\newcommand{\N}{\amsmathbb{N}}
\renewcommand{\Pr}{\mathsf{Pr}}
\newcommand{\E}{\mathsf{E}}
\newcommand{\Ex}{\mathsf{E}}
\newcommand{\Lap}{\mathsf{Lap}}

\newcommand{\cP}{\amsmathbb{P}}
\newcommand{\Z}{Z}
\newcommand{\X}{X}
\newcommand{\Y}{Y}
\newcommand{\M}{\mathcal{M}}
\newcommand{\cM}{\mathcal{M}}
\newcommand{\cN}{\mathcal{N}}
\renewcommand{\S}{\mathcal{S}}
\newcommand{\cS}{\mathcal{S}}

\newcommand{\tL}{\tilde{L}}
\newcommand{\tphi}{\tilde{\varphi}}
\newcommand{\rvX}{\mathbf{z}}
\newcommand{\rvY}{\mathbf{z}'}

\newcommand{\TV}{\mathsf{TV}}
\newcommand{\btwo}{\mathbb{2}}

\newtheorem{theorem}{Theorem}
\newtheorem{lemma}[theorem]{Lemma}

\newcommand{\supp}{\mathsf{supp}}
\newcommand{\norm}[1]{\| #1 \|}

\newcommand{\cU}{\mathcal{U}}

\newcommand{\Swo}{\S^{\mathsf{wo}}}
\newcommand{\Swr}{\S^{\mathsf{wr}}}
\newcommand{\Spo}{\S^{\mathsf{po}}}

\newcommand{\cR}{\mathcal{R}}
\newcommand{\one}{\amsmathbb{I}}

\usepackage[textsize=footnotesize,color=green!40]{todonotes}

\usepackage{authblk}

\title{Privacy Amplification by Subsampling: Tight Analyses via Couplings and Divergences}
\date{}
\author[1]{Borja Balle\footnote{Corresponding e-mail: \url{pigem@amazon.co.uk}}}
\affil[1]{Amazon Research, Cambridge, UK}
\author[2]{Gilles Barthe}
\affil[2]{IMDEA Software Institute, Madrid, Spain}
\author[3]{Marco Gaboardi}
\affil[3]{University at Buffalo (SUNY), Buffalo, USA}

\begin{document}
\maketitle

\begin{abstract}
  Differential privacy comes equipped with multiple analytical tools
  for the design of private data analyses. One important tool is the
  so-called ``privacy amplification by subsampling'' principle, which
  ensures that a differentially private mechanism run on a random
  subsample of a population provides higher privacy guarantees than
  when run on the entire population. Several instances of this
  principle have been studied for different random subsampling
  methods, each with an ad-hoc analysis.
  In this paper we present a general method that recovers and
  improves prior analyses, yields lower bounds and derives new
  instances of privacy amplification by subsampling. Our method
  leverages a characterization of differential privacy as a divergence
  which emerged in the program verification community. Furthermore, it
  introduces new tools, including advanced joint convexity
  and privacy profiles, which might be of independent interest.
\end{abstract}

\section{Introduction}
Subsampling is a fundamental tool in the design and analysis of
differentially private mechanisms. Broadly speaking, the intuition
behind the ``privacy amplification by subsampling'' principle  is
that the privacy guarantees of a differentially private mechanism can
be amplified by applying it to a small random subsample of records
from a given dataset. In machine learning, many classes of algorithms
involve sampling operations, e.g. stochastic optimization methods and
Bayesian inference algorithms, and it is not surprising that results quantifying
the privacy amplification obtained via subsampling play a key role in
designing differentially private versions of these learning algorithms
\citep{bassily2014private,wang2015privacy,abadi2016deep,DBLP:conf/uai/JalkoHD17,DBLP:journals/corr/ParkFCW16b,DBLP:journals/corr/ParkFCW16a}. 
Additionally, from a practical standpoint subsampling provides a straightforward method to obtain privacy amplification when the final mechanism is only available as a black-box. For example, in Apple's iOS and Google's Chrome deployments of differential privacy for data collection the privacy parameters are hard-coded into the implementation and cannot be modified by the user. In this type of settings, if the default privacy parameters are not satisfactory one could achieve a stronger privacy guarantee by devising a strategy that only submits to the mechanism a random sample of the data.

Despite the practical importance of subsampling, existing tools to bound privacy
amplification only work for specific forms of subsampling and
typically come with cumbersome proofs providing no information about the tightness of the resulting bounds. In this paper we remedy this situation by providing a general
framework for deriving tight privacy amplification results that can be
applied to any of the subsampling strategies considered in the
literature. Our framework builds on a characterization of
differential privacy in terms of
$\alpha$-divergences~\citep{barthe2013beyond}. This characterization has
been used before for program
verification~\citep{BartheKOB12,BartheGGHS16}, while we use it here for
the first time in the context of 
algorithm analysis. In order to do this, we develop several novel
analytical tools, including \emph{advanced joint convexity} -- a
property of $\alpha$-divergence with respect to mixture distributions --
and \emph{privacy profiles} -- a general tool describing the privacy guarantees
that private algorithms provide. 

One of our motivations to initiate a systematic study of privacy amplification by subsampling
is that this is an important primitive for the design of
differentially private algorithms which has received less attention
than other building blocks like composition theorems
\citep{dwork2010boosting,kairouz2017composition,murtagh2016complexity}. Given
the relevance of sampling operations in machine learning, it is
important to understand what are the limitations of privacy
amplification and develop a fine-grained understanding of its
theoretical properties. Our results provide a first step in this
direction by showing how privacy amplification resulting from different sampling techniques can be
analyzed by means of single set of tools, and by showing how these tools can
be used for proving lower bounds. Our analyses also highlight the importance of choosing a sampling technique that is well-adapted to the notion of neighbouring datasets under consideration.
A second motivation is that subsampling provides a natural example of
mechanisms where the output distribution is a mixture. Because
mixtures have an additive structure and differential privacy is
defined in terms of a multiplicative guarantee, analyzing the privacy
guarantees of mechanisms whose output distribution is a mixture is in
general a challenging task. Although our analyses are specialized to
mixtures arising from subsampling, we believe the tools we develop in
terms of couplings and divergences will also be useful to analyze other types of mechanisms involving mixture distributions.
Finally, we want to remark that privacy amplification results also play a role in analyzing the generalization and sample complexity properties of private learning algorithms \citep{kasiviswanathan2011can,beimel2013characterizing,bun2015differentially,wang2016learning}; an in-depth understanding of the interplay between sampling and differential privacy might also have applications in this direction.

\section{Problem Statement and Methodology Overview}

A \emph{mechanism} $\cM : X \to \cP(Z)$ with input space $X$ and output space $Z$ is a randomized
algorithm that on input $x$ outputs a sample from
the distribution $\cM(x)$ over $Z$. Here $\cP(Z)$ denotes the set of
probability measures on the output space $Z$. We implicitly assume $Z$
is equipped with a sigma-algebra of measurable subsets and a base
measure, in which case $\cP(Z)$ is restricted to probability measures
that are absolutely continuous with respect to the base measure. In
most cases of interest $Z$ is either a discrete space equipped with
the counting measure or an Euclidean space equipped with the Lebesgue
measure. We also assume $X$ is equipped with a binary symmetric
relation $\simeq_X$ defining the notion of neighbouring inputs.

Let $\varepsilon \geq 0$ and $\delta \in [0,1]$. A mechanism $\cM$ is
said to be \emph{$(\varepsilon,\delta)$-differentially private}
w.r.t.\, $\simeq_X$ if for every pair of inputs $x \simeq_X x'$
and every measurable subset $E \subseteq Z$ we have
\begin{align}
\Pr[\cM(x) \in E] \leq e^{\varepsilon} \Pr[\cM(x') \in E] + \delta \enspace.
\end{align}
For our purposes, it will be more convenient to express differential
privacy in terms of $\alpha$-divergences\footnote{Also known in the
literature as \emph{elementary
divergences} \citep{osterreicher2002csiszar} and \emph{hockey-stick
divergences} \citep{sason2016f}.}. Concretely, the $\alpha$-divergence ($\alpha \geq 1$) between two probability
measures $\mu, \mu' \in \cP(Z)$ is defined as\footnote{Here $d \mu /
d \mu'$ denotes the Radon-Nikodym derivative between $\mu$ and
$\mu'$. In particular, if $\mu$ and $\mu'$ have densities $p = d\mu /
d\nu$ and $p' = d\mu' / d\nu$ with respect to some base measure $\nu$,
then $d \mu / d \mu' = p / p'$.}
\begin{align}
D_{\alpha}(\mu \| \mu') = \sup_{E} \left(\mu(E) - \alpha \mu'(E)\right) =
\int_Z \left[\frac{d \mu}{d \mu'}(z) - \alpha\right]_+ d \mu'(z)
= 
\sum_{z \in Z} [\mu(z) - \alpha \mu'(z)]_+ \enspace,
\end{align}
where $E$ ranges over all measurable subsets of $Z$, $[\bullet]_+
= \max\{\bullet,0\}$, and the last equality is a specialization for
discrete $Z$. It is easy to see~\citep{barthe2013beyond} that $\M$ is
$(\varepsilon,\delta)$-differentially private if and only if
$D_{e^\varepsilon}(\M (x) \| \M (x'))\leq\delta$ for every $x$ and
$x'$ such that $x\simeq_X x'$.

In order to emphasize the relevant properties of $\cM$ from a privacy amplification point of
view, we introduce the concepts of \emph{privacy profile} and \emph{group-privacy
profiles}. The privacy profile $\delta_\cM$ of a mechanism $\cM$ is a
function associating to each privacy parameter $\alpha = e^{\varepsilon}$ a bound on the
$\alpha$-divergence between the results of running the mechanism on
two adjacent datasets, i.e.  $\delta_{\cM}(\varepsilon) = \sup_{x \simeq_X
x'} D_{e^{\varepsilon}}(\cM(x) \| \cM(x')) $ (we will discuss the properties of
this tool in more details in the next section).  Informally speaking,
the privacy profile represents the set of all of privacy parameters
under which a mechanism provides differential privacy.  In particular,
recall that an $(\varepsilon,\delta)$-DP mechanism $\cM$ is also
$(\varepsilon',\delta')$-DP for any $\varepsilon' \geq \varepsilon$
and any $\delta' \geq \delta$. The privacy profile $\delta_{\cM}$
defines a curve in $[0,\infty) \times [0,1]$ that separates the space
of privacy parameters into two regions: the ones for which $\cM$
satisfies differential privacy and the ones for which it does
not. This curve exists for every mechanism $\cM$, even for mechanisms
that satisfy pure DP for some value of $\varepsilon$. When the mechanism is clear from the context we might slightly abuse our notation and write $\delta(\varepsilon)$ or $\delta$ for the corresponding privacy profile.
To define group-privacy profiles $\delta_{\cM,k}$ ($k \geq 1$)
we use the path-distance $d$ induced by $\simeq_X$:
\begin{align*}
d(x,x') = \min \{ k : \exists x_1, \ldots, x_{k-1}, x \simeq_X x_1,
x_1 \simeq_X x_2, \ldots, x_{k-1} \simeq_X x' \} \enspace.
\end{align*}
With this notation, we define $\delta_{\cM,k}(\varepsilon)
= \sup_{d(x,x') \leq k} D_{e^{\varepsilon}}(\cM(x) \| \cM(x'))$. Note that
$\delta_{\cM} = \delta_{\cM,1}$.

\paragraph*{Problem Statement} A well-known method for
increasing privacy of a mechanism is to apply the mechanism to a
random subsample of the input database, rather than on the database
itself. Intuitively, the method decreases the chances of leaking
information about a particular individual because nothing about that
individual can be leaked in the cases where the individual is not
included in the subsample. The question addressed in this paper is to
devise methods for quantifying amplification and for proving
optimality of the bounds. This turns out to be a surprisingly subtle
problem.

Formally, let $\X$ and $\Y$ be two sets equipped with neighbouring relations
$\simeq_X$ and $\simeq_Y$ respectively. We assume that both $\X$ and
$\Y$ contain databases (modelled as sets, multisets, or tuples) over
a universe $\cU$ that represents all possible records contained in a
database. A subsampling mechanism is a randomized algorithm $\S :
\X \to \cP(\Y)$ that takes as input a database $x$ and outputs a
finitely supported distribution over datasets. Note that we find it
convenient to distinguish between $X$ and $Y$ because $x$ and $y$
might not always have the same type. For example, sampling with
replacement from a set $x$ yields a multiset $y$.

The problem of privacy amplification can now be stated as follows: let
$\M : \Y \to \cP(\Z)$ be a mechanism with privacy profile $\delta_{\cM}$ with respect to $\simeq_Y$, and
let $\S$ be a subsampling mechanism. Consider the subsampled mechanism $\M^{\S}
: \X \to \cP(\Z)$ given by $\M^{\S}(x) = \M(\S(x))$, where the composition notation means we feed a sample from $\S(x)$ into $\cM$.
The goal is to relate the privacy profiles of $\M$ and $\M^\S$, via an
inequality of the form: for every $\varepsilon \geq 0$, there exists
$0 \leq \varepsilon' \leq \varepsilon$ such that $\delta_{\M^\S}(\varepsilon') \leq
h(\delta_{\M}(\varepsilon))$, where $h$ is some function to be determined. In terms of differential privacy, one can be read as saying that if $\M$ is $(\varepsilon,\delta)$-DP, them the subsampled mechanism $\M^{\S}$ is $(\varepsilon',h(\delta))$-DP for some $\varepsilon' \leq \varepsilon$. This is a privacy amplification statement because the new mechanism has better privacy parameters than the original one.

A full specification of this problem requires formalizing the
following three ingredients: (i)
\emph{dataset representation} specifying whether the inputs to the
mechanism are sets, multisets, or tuples; (ii) \emph{neighbouring
relations} in $X$ and $Y$, including the usual remove/add-one
$\simeq_{r}$ and substitute-one $\simeq_{s}$ relations; (iii) \emph{subsampling method} and its parameters, with the most commonly used being subsample without replacement, subsampling with replacement, and Poisson subsampling.

Regardless of the specific setting being considered, the main
challenge in the analysis of privacy amplification by subsampling
resides in the fact that the output distribution of the mechanism $\mu
= \cM^{\cS}(x) \in \cP(Z)$ is a \emph{mixture distribution}.  In particular,
writing $\mu_y = \cM(y) \in \cP(Z)$ for any $y \in Y$ and taking $\omega
= \cS(x) \in \cP(Y)$ to be the (finitely supported) distribution over
subsamples from $x$ produced by the subsampling mechanism, we can
write $\mu = \sum_y \omega(y) \mu_y = \omega M$, where $M$ denotes the
Markov kernel operating on measures defined by $\cM$.  Consequently,
proving privacy amplifications results requires reasoning about the
mixtures obtained when sampling from two neighbouring datasets
$x \simeq_X x'$, and how the privacy parameters are affected by the mixture.

\begin{table}
\renewcommand{\arraystretch}{1.6}
\newcolumntype{C}{>{\centering\arraybackslash}X}{\small
\begin{tabularx}{\textwidth}{ccccCcc}
\toprule
Subsampling & $\simeq_Y$ & $\simeq_X$ & $\eta$ & $\delta'$ & Theorem \\
\midrule
Poisson($\gamma$) & R & R & $\gamma$ & $\gamma \delta$ & \ref{thm:po} \\
WOR($n$,$m$) & S & S & $\frac{m}{n}$ & $\frac{m}{n} \delta$ & \ref{thm:wo} \\
WR($n$,$m$) & S & S & $1 - \left(1 - \frac{1}{n}\right)^m$ & $\sum_{k=1}^m \binom{m}{k} \left(\frac{1}{n}\right)^k \left(1 - \frac{1}{n}\right)^{m-k} \delta_k$ & \ref{thm:wr} \\
WR($n$,$m$) & S & R & $1 - \left(1 - \frac{1}{n}\right)^m$ & $\sum_{k=1}^m \binom{m}{k} \left(\frac{1}{n}\right)^k \left(1 - \frac{1}{n}\right)^{m-k} \delta_k$ & \ref{thm:hybrid} \\
\bottomrule
\end{tabularx}
}
\vspace*{.7em}
\caption{Summary of privacy amplification bounds. Amplification parameter $\eta$: $e^{\varepsilon'} = 1 + \eta (e^\varepsilon -1)$. Types of subsampling: without replacement (WOR) and with replacement (WR). Neighbouring relations: remove/add-one (R) and substitute one (S).}\label{tab:results}
\end{table}

\paragraph{Our Contribution}
We provide a unified method for deriving privacy amplification by subsampling bounds (Section~\ref{sec:tools}).
Our method recovers all existing results in the literature and allow us to derive novel amplification bounds (Section~\ref{sec:examples}). In most cases our method also provides optimal constants which are shown to be tight by a generic lower bound (Section~\ref{sec:lb}).
Our analysis relies on properties of divergences and
privacy profiles, together with two additional ingredients.

The first ingredient is a novel \emph{advanced joint convexity}
property providing upper bounds on the
$\alpha$-divergence between overlapping mixture distributions. In the specific context
of differential privacy this result yields for every $x \simeq_X x'$:
\begin{align}\label{eqn:adj-preview}
D_{e^{\varepsilon'}}(\M^{\S} (x) \| \M^{\S} (x')) \leq \eta \cdot \left( (1-\beta) D_{e^{\varepsilon}}(\mu_1 \| \mu_0) + \beta
D_{e^{\varepsilon}} (\mu_1 \| \mu_1') \right) \enspace,
\end{align}
for $e^{\varepsilon'} = 1 + \eta (e^{\varepsilon} - 1)$, some $\beta
\in [0,1]$, and $\eta = \TV(\cS(x),\cS(x'))$ being the total
variation distance between the distributions over subsamples. Here $\mu_0,\mu_1,\mu'_1 \in \cP(Z)$ are suitable measures obtained from $\M^{\S}(x)$ and $\M^{\S}(x')$ through a coupling and projection operation. In particular, the proof of advanced
joint convexity uses ideas from probabilistic couplings, and more
specifically the maximal coupling construction (see Theorem~\ref{thm:ajc} and its proof for more details).
It is also interesting to note that the non-linear relation $\varepsilon' = \log(1 +
\eta (e^{\varepsilon} - 1))$ already appears in some existing
privacy amplification results (e.g.~\cite{li2012sampling}). Although for small $\varepsilon$ and $\eta$ this relation yields $\varepsilon' = O(\eta \varepsilon)$, our results show that the more complicated non-linear relation is in fact a fundamental aspect of privacy amplification by subsampling.

The second ingredient in our analysis establishes an upper bound for
the divergences occurring in the right hand side of \eqref{eqn:adj-preview}
in terms of group-privacy profiles. It states that under suitable conditions,
we have $D_{e^{\varepsilon}}(\nu M \| \nu' M) \leq \sum_{k\geq 1} \lambda_k(\nu) \delta_{\cM,k}(e^{\varepsilon})$
for suitable choices of $\lambda_k$. Again, the proof of the
inequality uses tools from probabilistic couplings.

The combination of these results yields a bound of the privacy profile
of $\M^{\S}$ as a function of the group-privacy profiles of
$\M$. Based on this inequality, we will establish several privacy
amplification result and prove tightness results.  This methodology
can be applied to any of the settings discussed above in terms of
dataset representation, neighbouring relation, and type of
subsampling. Table~\ref{tab:results} summarizes several results that can be obtained with our method (see Section~\ref{sec:examples} for details). The supplementary material also contains plots illustrating our bounds (Figure~\ref{fig:profiles}) and proofs of all the results presented in the paper.

\section{Tools: Couplings, Divergences and Privacy Profiles}\label{sec:tools}
We next introduce several tools that will be used to support our
analyses. The first and second tools are known, whereas the remaining
tools are new and of independent interest.

\paragraph*{Divergences}
The following
characterization follows immediately from the definition of
$\alpha$-divergence in terms of the supremum over $E$.
\begin{theorem}[\citep{barthe2013beyond}]
A mechanism $\cM$ is $(\varepsilon,\delta)$-differentially private
with respect to $\simeq_X$ if and only if $\sup_{x \simeq_X x'}
D_{e^{\varepsilon}}(\cM(x) \| \cM(x')) \leq \delta$.
\end{theorem}
Note that in the statement of the theorem we take $\alpha =
e^{\varepsilon}$. Throughout the paper we sometimes use these two notations
interchangeably to make expressions more compact.

We now state consequences of the definition of
$\alpha$-divergence: (i) $0 \leq D_{\alpha}(\mu \| \mu') \leq 1$; (ii)
the function $\alpha \mapsto D_{\alpha}(\mu \| \mu')$ is monotonically
decreasing; (iii) the function $(\mu,\mu') \mapsto
D_{\alpha}(\mu \| \mu')$ is jointly convex. Furthermore, one can show
that $\lim_{\alpha \to \infty} D_{\alpha}(\mu \| \mu') = 0$ if and only if $\supp(\mu) \subseteq \supp(\mu')$.

\paragraph*{Couplings}
Couplings are a standard tool for deriving upper bounds for the
statistical distance between distributions. Concretely, it is
well-known that the total variation distance between two distributions
$\nu, \nu' \in \cP(Y)$ satisfies $\TV (\nu,\nu') \leq \Pr_{\pi}[y \neq
y']$ for any coupling $\pi$, where equality is attained by taking the
so-called maximal coupling. We recall the definition of coupling and
provide a construction of the maximal coupling, which we shall use in
later sections.

A coupling between two distributions $\nu, \nu' \in \cP(Y)$ is a
distribution $\pi \in \cP(Y \times Y)$ whose marginals along the
projections $(y,y') \mapsto y$ and $(y,y') \mapsto y'$ are $\nu$ and
$\nu'$ respectively. Couplings always exist, and furthermore, there
exists a maximal coupling, which exactly characterizes the total variation
distance between $\nu$ and $\nu'$. Let $\nu_0
(y)= \min\{\nu(y),\nu'(y)\}$ and let $\eta =
\TV (\nu,\nu') = 1 - \sum_{y \in Y} \nu_0(y)$, where $\TV$ denotes the
total variation distance.
The maximal coupling between $\nu$ and
$\nu'$ is defined as the mixture $\pi =
(1-\eta) \pi_0 + \eta \pi_1$, where $\pi_0 (y,y')=
\nu_0(y) \mathbb{1}[y=y'] /(1-\eta)$, and
$\nu_1 (y,y')=(\nu(y)-\nu_0(y)) (\nu'(y')-\nu_0(y'))/\eta$. Projecting the maximal coupling along the marginals yields the overlapping mixture decompositions $\nu = (1-\eta) \nu_0 + \eta \nu_1$ and $\nu' = (1-\eta) \nu_0 + \eta \nu_1'$.

\paragraph{Advanced Joint Convexity}
The privacy amplification phenomenon is tightly connected to an
interesting new form of joint convexity for $\alpha$-divergences, which we call advanced joint
convexity.

\begin{theorem}[Advanced Joint Convexity of $D_{\alpha}$\footnote{Proofs of all our results are presented in the appendix.}]\label{thm:ajc}
Let $\mu, \mu' \in \cP(Z)$ be measures satisfying $\mu = (1-\eta) \mu_0 + \eta \mu_1$ and $\mu' = (1-\eta) \mu_0 + \eta \mu_1'$ for some $\eta$, $\mu_0$, $\mu_1$, and $\mu_1'$.
Given $\alpha \geq 1$, let $\alpha' = 1 + \eta (\alpha-1)$ and
$\beta = \alpha' / \alpha$. Then the following holds:
\begin{align}\label{eqn:ajc}
D_{\alpha'}(\mu \| \mu') = \eta D_{\alpha}(\mu_1 \| (1-\beta) \mu_0 + \beta \mu_1') \enspace.
\end{align}
\end{theorem}
Note that writing $\alpha = e^{\varepsilon}$ and $\alpha' =
e^{\varepsilon'}$ in the above lemma we get the relation $\varepsilon'
= \log(1 + \eta(e^{\varepsilon} - 1))$. Applying standard joint convexity to the right hand side above we conclude:
$D_{\alpha'}(\mu \| \mu') \leq (1-\beta) \eta D_{\alpha}(\mu_1 \|  \mu_0)
+ \beta \eta D_{\alpha}(\mu_1 \| \mu_1')$.
Note that applying joint convexity directly on $D_{\alpha'}(\mu \| \mu')$ instead of advanced joint
complexity yields a weaker bound which implies amplification for the
$\delta$ privacy parameter, but not for the $\varepsilon$ privacy
parameter.

When using advanced joint convexity to analyze privacy amplification we consider two elements $x$ and $x'$ and fix the following notation.
Let $\omega=\S(x)$ and $\omega'=\S(x')$ and $\mu=\omega M$ and
$\mu'=\omega' M$, where we use the notation $M$ to denote the Markov
kernel associated with mechanism $\cM$ operating on measures over $Y$.
We then consider the mixture factorization of $\omega$ and $\omega'$ obtained by taking the
decompositions induced by projecting the \emph{maximal coupling} $\pi = (1 - \eta) \pi_0 + \eta \pi_1$ on the first and second marginals: $\omega = (1-\eta) \omega_0 + \eta \omega_1$ and
$\omega' = (1-\eta) \omega_0 + \eta \omega_1'$. It is easy to see from the construction of the maximal coupling that $\omega_1$ and $\omega_1'$ have disjoint supports and $\eta$ is the smallest probability such that this condition holds.
In this way we obtain the canonical mixture decompositions $\mu = (1-\eta) \mu_0 + \eta \mu_1$ and
$\mu' = (1-\eta) \mu_0 + \eta \mu_1'$, where $\mu_0 = \omega_0 M$,
$\mu_1 = \omega_1 M$ and $\mu_1' = \omega_1' M$.

\paragraph{Privacy Profiles}
We state some important properties of privacy profiles. Our first result illustrates our claim that the
``privacy curve'' exists for every mechanism $\cM$ in the context of the Laplace
output perturbation mechanism.

\begin{theorem}\label{thm:profileLap}
Let $f : X \to \R$ be a function with global sensitivity $\Delta = \sup_{x \simeq_X x'} |f(x) - f(x')|$. Suppose $\cM(x) = f(x) + \Lap(b)$ is a Laplace output perturbation mechanism with noise parameter $b$. The privacy profile of $\cM$ is given by $\delta_{\cM}({\varepsilon}) = [1 - \exp(\frac{\varepsilon - \theta}{2})]_+$, where $\theta = \Delta / b$.
\end{theorem}
The well-known fact that the Laplace mechanism with $b \geq \Delta / \varepsilon$ is $(\varepsilon,0)$-DP follows from this result by noting that $\delta_{\cM}({\varepsilon}) = 0$ for any $\varepsilon \geq \theta$. However, Theorem~\ref{thm:profileLap} also provides more information: it shows that for $\varepsilon < \Delta/b$ the Laplace mechanism with noise parameter $b$ satisfies $(\varepsilon,\delta)$-DP with $\delta = \delta_{\cM}({\varepsilon})$.

For mechanisms that only satisfy approximate DP, the privacy profile provides information about the behaviour of $\delta_{\cM}({\varepsilon})$ as we increase $\varepsilon \to \infty$. The classical analysis for the Gaussian output perturbation mechanism provides some information in this respect. Recall that for a function $f : X \to \R^d$ with $L_2$ global sensitivity $\Delta = \sup_{x \simeq_X x'} \norm{f(x) - f(x)}_2$ the mechanism $\cM(x) = f(x) + \cN(0,\sigma^2 I)$ satisfies $(\varepsilon,\delta)$-DP if $\sigma^2 \geq 2 \Delta^2 \log(1.25/\delta) / \varepsilon^2$ and $\varepsilon \in (0,1)$ (cf.\ \citep[Theorem A.1]{dwork2014algorithmic}). This can be rewritten as $\delta_{\cM}({\varepsilon}) \leq 1.25 e^{- \varepsilon^2 / 2 \theta^2}$ for $\varepsilon \in (0,1)$, where $\theta = \Delta / \sigma$.
Recently, Balle and Wang \citep{BWicml17agm} gave a new analysis of the Gaussian mechanism that is valid for all values of $\varepsilon$. Their analysis can be interpreted as providing an expression for the privacy profile of the Gaussian mechanism in terms of the CDF of a standard normal distribution $\Phi(t) = (2\pi)^{-1/2} \int_{-\infty}^t e^{-r^2/2} dr$.

\begin{theorem}[\citep{BWicml17agm}]\label{thm:profileGau}
Let $f : X \to \R^d$ be a function with $L_2$ global sensitivity $\Delta$. For any $\sigma > 0$ let $\theta = \Delta / \sigma$.
The privacy profile of the Gaussian mechanism $\cM(x) = f(x) + \cN(0,\sigma^2 I)$ is given by $\delta_{\cM}(e^{\varepsilon}) = \Phi(\theta/2 - \varepsilon/ \theta) - e^{\varepsilon} \Phi( -\theta/2 - \varepsilon / \theta)$.

\end{theorem}

Interestingly, the proof of Theorem~\ref{thm:profileGau} implicitly provides a characterization of privacy profiles in terms of privacy loss random variables that holds for any mechanism. Recall that the \emph{privacy loss random variable} of a mechanism $\cM$ on inputs $x \simeq_X x'$ is defined as $L_{\cM}^{x,x'} = \log (d\mu / d\mu')(\mathbf{z})$, where $\mu = \cM(x)$, $\mu' = \cM(x')$, and $\mathbf{z} \sim \mu$.

\begin{theorem}[\citep{BWicml17agm}]\label{thm:profileL}
The privacy profile of any mechanism $\cM$ satisfies
\begin{align*}
\delta_{\cM}({\varepsilon}) = \sup_{x \simeq_X x'} \left( \Pr[L_{\cM}^{x,x'} > \varepsilon] - e^{\varepsilon} \Pr[L_{\cM}^{x',x} < - \varepsilon] \right) \enspace.
\end{align*}
\end{theorem}

The characterization above generalizes the well-known inequality $\delta_{\cM}({\varepsilon}) \leq \sup_{x \simeq_X x'} \Pr[L_{\cM}^{x,x'} > \varepsilon]$ (eg.\ see \citep{dwork2014algorithmic}). This bound is often used to derive $(\varepsilon,\delta)$-DP guarantees from other notions of privacy defined in terms of the moment generating function of the privacy loss random variable, including concentrated DP \citep{dwork2016concentrated}, zero-concentrated DP \citep{DBLP:conf/tcc/BunS16}, R{\'e}nyi DP \citep{DBLP:conf/csfw/Mironov17}, and truncated concentrated DP \citep{tcdp}.
We now show a reverse implication also holds. Namely, that privacy profiles can be used to recover all the information provided by the moment generating function of the privacy loss random variable.

\begin{theorem}\label{thm:PPtoMGF}
Given a mechanism $\cM$ and inputs $x \simeq_X x'$ let $\mu = \cM(x)$ and $\mu' = \cM(x')$. For $s \geq 0$, define the moment generating function $\varphi_{\cM}^{x,x'}(s) = \E[\exp(s L_{\cM}^{x,x'})]$. Then we have
\begin{align*}
\varphi_{\cM}^{x,x'}(s) = 1 + s (s+1) \int_{0}^{\infty} \left( e^{s \varepsilon} D_{e^\varepsilon}(\mu \| \mu') + e^{-(s+1) \varepsilon} D_{e^\varepsilon}(\mu' \| \mu) \right) d \varepsilon \enspace. 
\end{align*}
In particular, if $D_{e^\varepsilon}(\mu \| \mu') = D_{e^\varepsilon}(\mu' \| \mu)$ holds\footnote{For example, this is satisfied by all output perturbation mechanisms with symmetric noise distributions.} for every $x \simeq_X x'$, then $\sup_{x \simeq_X x'} \varphi_{\cM}^{x,x'}(s) = 1 + s (s+1) \int_{0}^{\infty} (e^{s \varepsilon} + e^{-(s+1) \varepsilon}) \delta_{\cM}({\varepsilon}) d \varepsilon$.
\end{theorem}

\paragraph{Group-privacy Profiles}

Recall the $k$th group privacy profile of a mechanism $\M$ is defined as
$\delta_{\cM,k}(\varepsilon) = \sup_{d(x,x') \leq k} D_{e^{\varepsilon}}(\cM(x) \| \cM(x'))$.
A standard group privacy analysis\footnote{If $\cM$ is $(\varepsilon,\delta)$-DP with respect to $\simeq_Y$, then it is $(k \varepsilon, ((e^{k \varepsilon}-1)/(e^{\varepsilon}-1)) \delta)$-DP with respect to $\simeq_Y^k$, cf.\ \citep[Lemma 2.2]{DBLP:books/sp/17/Vadhan17}} immediately yields $\delta_{\cM,k}(\varepsilon) \leq (e^{\varepsilon}-1) \delta_{\cM}(\varepsilon/k)/(e^{\varepsilon/k}-1)$.
However, ``white-box'' approaches based on full knowledge of the privacy profile of $\cM$ can be used to improve this result for specific mechanisms. For example, it is not hard to see that, combining the expressions from Theorems~\ref{thm:profileLap} and~\ref{thm:profileGau} with the triangle inequality on the global sensitivity of changing $k$ records in a dataset, one obtains bounds that improve on the ``black-box'' approach for all ranges of parameters for the Laplace and Gaussian mechanisms. This is one of the reasons why we state our bounds directly in terms of (group-)privacy profiles (a numerical comparison can be found in the supplementary material).

\paragraph{Distance-compatible Coupling} The last tool we need to prove general privacy amplification bounds based on $\alpha$-divergences is the existence of a certain type of couplings between two distributions like the ones occurring in the right hand side of \eqref{eqn:ajc}. Recall that any coupling $\pi$ between two distributions $\nu, \nu' \in \cP(Y)$ can be used to rewrite the mixture distributions $\tilde{\mu} = \nu M$ and $\tilde{\mu}' = \nu' M$ as $\tilde{\mu} = \sum_{y,y'} \pi_{y,y'} \cM(y)$ and $\tilde{\mu}' = \sum_{y,y'} \pi_{y,y'} \cM(y')$. Using the joint convexity of $D_{\alpha}$ and the definition of group-privacy profiles to get the bound
\begin{align}\label{eqn:anyPi}
D_{e^{\varepsilon}}(\tilde{\mu} \| \tilde{\mu}')
\leq
\sum_{y,y'} \pi_{y,y'} D_{e^{\varepsilon}}(\cM(y) \| \cM(y'))
\leq
\sum_{y,y'} \pi_{y,y'} \delta_{\cM,d_Y(y,y')}(\varepsilon)
\enspace.
\end{align}
Since this bound holds for any coupling $\pi$, one can set out to optimize it by finding a coupling the minimizes the right hand side of \eqref{eqn:anyPi}. We show that the existence of couplings whose support is contained inside a certain subset of $Y \times Y$ is enough to obtain an optimal bound. Furthermore, we show that when this condition is satisfied the resulting bound depends only on $\nu$ and the group-privacy profiles of $\cM$.

We say that two distributions $\nu, \nu' \in \cP(Y)$ are \emph{$d_Y$-compatible} if there exists a coupling $\pi$ between $\nu$ and $\nu'$ such for any $(y,y') \in \supp(\pi)$ we have $d_Y(y,y') = d_Y(y,\supp(\nu'))$, where the distance between a point $y$ and the set $\supp(\nu')$ is defined as the distance between $y$ and the closest point in $\supp(\nu')$.

\begin{theorem}\label{thm:dcc}
Let $C(\nu,\nu')$ be the set of all couplings between $\nu$ and $\nu'$ and for $k \geq 1$ let $Y_k = \{ y \in \supp(\nu) : d_Y(y,\supp(\nu')) = k \}$.
If $\nu$ and $\nu'$ are $d_Y$-compatible, then the following holds:
\begin{align}\label{eqn:nocoupling}
\min_{\pi \in C(\nu,\nu')} \sum_{y,y'} \pi_{y,y'} \delta_{\cM,d_Y(y,y')}(\varepsilon) = \sum_{k \geq 1} \nu(Y_k) \delta_{\cM,k}(\varepsilon) \enspace.
\end{align}
\end{theorem}

Applying this result to the bound resulting from the right hand side of \eqref{eqn:ajc} yields most of the concrete privacy amplification results presented in the next section.

\section{Privacy Amplification Bounds}\label{sec:examples}

In this section we provide explicit privacy amplification bounds for the most common subsampling methods and neighbouring relations found in the literature on differential privacy, and provide pointers to existing bounds and other related work.
For our analysis we work with order-independent representations of datasets without repetitions, i.e.\ sets. This is mostly for technical convenience, since all our results also hold if one considers datasets represented as tuples or multisets. Note however that subsampling with replacement for a set can yield a multiset; hence we introduce suitable notations for sets and multisets.

Fix a universe of records $\cU$ and let $\btwo = \{0,1\}$.
We write $\btwo^{\cU}$ and $\N^{\cU}$ for the spaces of all sets and multisets with records from $\cU$. Note every set is also a multiset. For $n \geq 0$ we also write $\btwo_n^{\cU}$ and $\N_n^{\cU}$ for the spaces of all sets and multisets containing exactly $n$ records\footnote{In the case of multisets records are counted with multiplicity.} from $\cU$. Given $x \in \N^{\cU}$ we write $x_u$ for the number of occurrences of $u \in \cU$ in $x$. The support of a multiset $x$ is the defined as the set $\supp(x) = \{ u \in \cU : x_u > 0 \}$ of elements that occur at least once in $x$. Given multisets $x, x' \in \N^{\cU}$ we write $x' \subseteq x$ to denote that $x'_u \leq x_u$ for all $u \in \cU$.

For order-independent datasets represented as multisets it is natural to consider the two following neighbouring relations. The \emph{remove/add-one} relation is obtained by letting $x \simeq_{r} x'$ hold whenever $x \subseteq x'$ with $|x| = |x'| - 1$ or $x' \subseteq x$ with $|x| = |x'| + 1$; i.e.\ $x'$ is obtained by removing or adding a single element to $x$. The \emph{substitute-one} relation is obtained by letting $x \simeq_{s} x'$ hold whenever $\norm{x - x'}_1 = 2$ and $|x| = |x'|$; i.e.\ $x'$ is obtained by replacing an element in $x$ with a different element from $\cU$. Note how $\simeq_r$ relates pairs of datasets with different sizes, while $\simeq_s$ only relates pairs of datasets with the same size.

\paragraph{Poisson Subsampling}
Perhaps the most well-known privacy amplification result refers to the analysis of Poisson subsampling with respect to the remove/add-one relation. In this case the subsampling mechanism $\Spo_\gamma : \btwo^{\cU} \to \cP(\btwo^{\cU})$ takes a set $x$ and outputs a sample $y$ from the distribution $\omega = \Spo_\gamma(x)$ supported on all set $y \subseteq x$ given by $\omega(y) = \gamma^{|y|} (1-\gamma)^{|x|-|y|}$. This corresponds to independently adding to $y$ with probability $\gamma$ each element from $x$.
Now, given a mechanism $\cM : \btwo^{\cU} \to \cP(Z)$ with privacy profile $\delta_{\cM}$ with respect to $\simeq_r$, we are interested in bounding the privacy profile of the subsampled mechanism $\cM^{\Swo_\gamma}$ with respect to $\simeq_r$.

\begin{theorem}\label{thm:po}
Let $\M' = \cM^{\Spo_\gamma}$.
For any $\varepsilon \geq 0$ we have $\delta_{\cM'}(\varepsilon') \leq \gamma \delta_{\cM}(\varepsilon)$, where $\varepsilon' = \log(1 + \gamma (e^\varepsilon-1))$.
\end{theorem}

Privacy amplification with Poisson sampling was used in \citep{chaudhuri2006random,beimel2010bounds,kasiviswanathan2011can,beimel2014bounds}, which considered loose bounds. A proof of this tight result in terms of $(\varepsilon,\delta)$-DP was first given in \citep{li2012sampling}. In the context of the moments accountant technique based on the moment generating function of the privacy loss random variable, \citep{abadi2016deep} provide an amplification result for Gaussian output perturbation mechanisms under Poisson subsampling.

\paragraph{Sampling Without Replacement} Another known results on privacy amplification corresponds to the analysis of
sampling without replacement with respect to the substitution relation. In this case one considers the subsampling mechanism $\Swo_m : \btwo_n^{\cU} \to \cP(\btwo_m^{\cU})$ that given a set $x \in \btwo_n^{\cU}$ of size $n$ outputs a sample from the uniform distribution $\omega = \Swo_m(x)$ over all subsets $y \subseteq x$ of size $m \leq n$. Then, for a given a mechanism $\cM : \btwo_m^{\cU} \to \cP(Z)$ with privacy profile $\delta_{\cM}$ with respect to the substitution relation $\simeq_s$ on sets of size $m$, we are interested in bounding the privacy profile of the mechanism $\cM^{\Swo_m}$ with respect to the substitution relation on sets of size $n$.

\begin{theorem}\label{thm:wo}
Let $\M' = \cM^{\Swo_m}$.
For any $\varepsilon \geq 0$ we have $\delta_{\cM'}(\varepsilon') \leq (m/n) \delta_{\cM}(\varepsilon)$, where $\varepsilon' = \log(1 + (m/n)(e^{\varepsilon}-1))$.
\end{theorem}

This setting has been used in \citep{beimel2013characterizing,bassily2014private,wang2016learning} with non-tight bounds.
A proof of this tight bound formulated in terms of $(\varepsilon,\delta)$-DP can be directly recovered from Ullman's class notes \citep{ullmanclass}, although the stated bound is weaker.
R\'enyi DP amplification bounds for subsampling without replacement were developed in \citep{wang2018subsampled}.

\paragraph{Sampling With Replacement}
Next we consider the case of sampling with replacement with respect to the substitution relation $\simeq_s$. The subsampling with replacement mechanism $\Swr_m : \btwo_n^{\cU} \to \cP(\N_m^{\cU})$ takes a set $x$ of size $n$ and outputs a sample from the multinomial distribution $\omega = \Swr_m(x)$ over all multisets $y$ of size $m \leq n$ with $\supp(y) \subseteq x$, given by $\omega(y) = (m!/n^m) \prod_{u \in \cU} x_u / (y_u!)$.
In this case we suppose the base mechanism $\cM : \N_m^{\cU} \to \cP(Z)$ is defined on multisets and has privacy profile $\delta_{\cM}$ with respect to $\simeq_s$.
We are interested in bounding the privacy profile of the subsampled mechanism $\cM^{\Swr_m} : \btwo_n^{\cU} \to \cP(Z)$ with respect to $\simeq_s$.

\begin{theorem}\label{thm:wr}
Let $\M' = \cM^{\Swr_m}$.
Given $\varepsilon \geq 0$ and $\varepsilon' = \log(1 + (1 - (1-1/n)^m) (e^\varepsilon - 1))$ we have
\begin{align*}
\delta_{\cM'}(\varepsilon') \leq 
\sum_{k=1}^m \binom{m}{k} \left(\frac{1}{n}\right)^k \left(1 - \frac{1}{n}\right)^{m-k} \delta_{\cM,k}(\varepsilon) \enspace.
\end{align*}
\end{theorem}

Note that if $m = \gamma n$, then $1 - (1-1/n)^m \approx \gamma$.
A version of this bound in terms of $(\varepsilon,\delta)$-DP that implicitly uses the group privacy property can be found in \citep{bun2015differentially}. Our bound matches the asymptotics of \citep{bun2015differentially} while providing optimal constants and allowing for white-box group privacy bounds.

\paragraph{Hybrid Neighbouring Relations}
Using our method it is also possible to analyze new settings which have not been considered before. One interesting example occurs when there is a mismatch between the two neighbouring relations arising in the analysis. For example, suppose one knows the group-privacy profiles $\delta_{\cM,k}$ of a base mechanism $\cM : \N_m^{\cU} \to \cP(Z)$ with respect to the substitution relation $\simeq_s$.
In this case one could ask whether it makes sense to study the privacy profile of the subsampled mechanism $\cM^{\Swr_m} : \btwo^{\cU} \to \cP(Z)$ with respect to the remove/add relation $\simeq_r$.
In principle, this makes sense in settings where the size of the inputs to $\cM$ is restricted due to implementation constraints (eg.\ limited by the memory available in a GPU used to run a private mechanism that computes a gradient on a mini-batch of size $m$).
In this case one might still be interested in analyzing the privacy loss incurred from releasing such stochastic gradients under the remove/add relation. Note that this setting cannot be implemented using sampling without replacement since under the remove/add relation we cannot a priori guarantee that the input dataset will have at least size $m$ because the size of the dataset must be kept private \citep{DBLP:books/sp/17/Vadhan17}.
Furthermore, one cannot hope to get a meaningful result about the privacy profile of the subsampled mechanism across all inputs sets in $\btwo^{\cU}$; instead the privacy guarantee will depend on the size of the input dataset as shown in the following result.

\begin{theorem}\label{thm:hybrid}
Let $\M' = \cM^{\Swr_m}$.
For any $\varepsilon \geq 0$ and $n \geq 0$ we have
\begin{align*}
\sup_{x \in \btwo_n^{\cU}, x \simeq_r x'} D_{e^{\varepsilon'}}(\cM'(x) \| \cM'(x'))
\leq
\sum_{k=1}^m \binom{m}{k} \left(\frac{1}{n}\right)^{k} \left(1 - \frac{1}{n}\right)^{m-k} \delta_{\cM,k}(\varepsilon) \enspace,
\end{align*}
where $\varepsilon' = \log(1 + (1 - (1-1/n)^m) (e^{\varepsilon} - 1))$.
\end{theorem}

\paragraph{When the Neighbouring Relation is ``Incompatible''}
Now we consider a simple example where distance-compatible couplings are not available: Poisson subsampling with respect to the substitution relation. Suppose $x, x' \in \btwo_n^{\cU}$ are sets of size $n$ related by the substitution relation $\simeq_s$. Let $\omega = \Spo_\eta(x)$ and $\omega' = \Spo_\eta(x')$ and note that $\TV(\omega,\omega') = \eta$. Let $x_0 = x \cap x'$ and $v = x \setminus x_0$, $v' = x' \setminus x_0$. In this case the factorization induced by the maximal coupling is obtained by taking $\omega_0 = \Spo_\eta(x_0)$, $\omega_1(y \cup \{v\}) = \omega_0(y)$, and $\omega_1'(y \cup \{v'\}) = \omega_0(y)$. Now the support of $\omega_0$ contains sets of sizes between $0$ and $n-1$, while the supports of $\omega_1$ and $\omega_1$ contain sets of sizes between $1$ and $n$. From this observation one can deduce that $\omega_1$ and $\omega_0$ are not $d_{\simeq_s}$-compatible, and $\omega_1$ and $\omega_1'$ are not $d_{\simeq_r}$-compatible.

This argument shows that the method we used to analyze the previous settings cannot be extended to analyze Poisson subsampling under the substitution relation, regardless of whether the privacy profile of the base mechanism is given in terms of the replacement/addition or the substitution relation. This observation is saying that some pairings between subsampling method and neighbouring relation are more natural than others. Nonetheless, even without distance-compatible couplings it is possible to provide privacy amplification bounds for Poisson subsampling with respect to the substitution relation, although the resulting bound is quite cumbersome. The corresponding statement and analysis can be found in the supplementary material.

\section{Lower Bounds}\label{sec:lb}

In this section we show that many of the results given in the previous section are tight by constructing a randomized membership mechanism that attains these upper bounds. For the sake of generality, we state the main construction in terms of tuples instead of multisets. In fact, we prove a general lemma that can be used to obtain tightness results for any subsampling mechanism and any neighbouring relation satisfying two natural assumptions.

For $p \in [0,1]$ let $\cR_p : \{0,1\} \to \cP(\{0,1\})$ be the randomized response mechanism that given $b \in \{0,1\}$ returns $b$ with probability $p$ and $1 - b$ with probability $1 - p$. Note that for $p = (e^\varepsilon + \delta)/(e^{\varepsilon} + 1)$ this mechanism is $(\varepsilon,\delta)$-DP.
Let $\nu_0 = \cR_p(0)$ and $\nu_1 = \cR_p(1)$. For any $\varepsilon \geq 0$ and $p \in [0,1]$ define $\psi_p(\varepsilon) = [p - e^{\varepsilon} (1 - p)]_+$. It is easy to verify that $D_{e^{\varepsilon}}(\nu_0 \| \nu_1) = D_{e^{\varepsilon}}(\nu_1 \| \nu_0) = \psi_p(\varepsilon)$.
Now let $\cU$ be a universe containing at least two elements. For $v \in \cU$ and $p \in [0,1]$ we define the \emph{randomized membership} mechanism $\cM_{v,p}$ that given a tuple $x = (u_1, \ldots, u_n) \in \cU^{\star}$ returns $\cM_{v,p}(x) = \cR_p(\one[v \in x])$.
We say that a subsampling mechanism $\cS : X \to \cP(\cU^{\star})$ defined on some set $X \subseteq \cU^{\star}$ is \emph{natural} if the following two conditions are satisfied: (1) for any $x \in X$ and $u \in \cU$, if $u \in x$ then there exists $y \in \supp(\cS(x))$ such that $u \in y$; (2) for any $x \in X$ and $u \in \cU$, if $u \notin x$ then we have $u \notin y$ for every $y \in \supp(\cS(x))$.

\begin{lemma}\label{lem:naturalS}
Let $X \subseteq \cU^{\star}$ be equipped with a neighbouring relation $\simeq_X$ such that there exist $x \simeq_X x'$ with $v \in x$ and $v \notin x'$. Suppose $\cS : X \to \cP(\cU^{\star})$ is a natural subsampling mechanism and let
$\eta = \sup_{x \simeq_X x'} \TV(\cS(x),\cS(x'))$.
For any $\varepsilon \geq 0$ and $\varepsilon' = \log(1 + \eta (e^{\varepsilon}-1))$ we have
\begin{align*}
\delta_{\cM_{v,p}^{\cS}}(\varepsilon') = \sup_{x \simeq_X x'} D_{e^{\varepsilon'}}(\cM_{v,p}^{\cS}(x) \| \cM_{v,p}^{\cS}(x')) = \eta \psi_p(\varepsilon) \enspace.
\end{align*}
\end{lemma}

We can now apply this lemma to show that the first three results from previous section are tight. This requires specializing from tuples to (multi)sets, and plugging in the definitions of neighbouring relation, subsampling mechanism, and $\eta$ used in each of these theorems.

\begin{theorem}
The mechanism $\cM_{v,p}$ attains the bounds in Theorems~\ref{thm:po},~\ref{thm:wo},~\ref{thm:wr} for any $p$ and $\eta$.
\end{theorem}

\section{Conclusions}
We have developed a general method for reasoning about privacy
amplification by subsampling. Our method is applicable to many
different settings, some which have already been studied in the
literature, and others which are new. Technically, our method
leverages two new tools of independent interest: advanced joint
convexity and privacy profiles.
In the future, it would be interesting to study whether our tools can
be extended to give concrete bounds on privacy amplification  for other privacy
notions such as concentrated DP \citep{dwork2016concentrated}, zero-concentrated DP \citep{DBLP:conf/tcc/BunS16}, R{\'e}nyi DP \citep{DBLP:conf/csfw/Mironov17}, and truncated concentrated DP \citep{tcdp}. A good starting point is Theorem~\ref{thm:PPtoMGF} establishing
relations between privacy profiles and moment generating functions of
the privacy loss random variable. An alternative approach is to extend the recent results for R{\'e}nyi DP amplification by subsampling without replacement given in \citep{wang2018subsampled} to more general notions of subsampling and neighbouring relations.

\subsubsection*{Acknowledgments}
This research was initiated during the 2017 Probabilistic Programming Languages workshop hosted by McGill University's Bellairs Research Institute.

\bibliographystyle{plainnat}
\bibliography{main}

\begin{thebibliography}{31}
\providecommand{\natexlab}[1]{#1}
\providecommand{\url}[1]{\texttt{#1}}
\expandafter\ifx\csname urlstyle\endcsname\relax
  \providecommand{\doi}[1]{doi: #1}\else
  \providecommand{\doi}{doi: \begingroup \urlstyle{rm}\Url}\fi

\bibitem[Abadi et~al.(2016)Abadi, Chu, Goodfellow, McMahan, Mironov, Talwar,
  and Zhang]{abadi2016deep}
Mart{\'\i}n Abadi, Andy Chu, Ian Goodfellow, H~Brendan McMahan, Ilya Mironov,
  Kunal Talwar, and Li~Zhang.
\newblock Deep learning with differential privacy.
\newblock In \emph{Proceedings of the 2016 ACM SIGSAC Conference on Computer
  and Communications Security}, pages 308--318. ACM, 2016.

\bibitem[Balle and Wang(2018)]{BWicml17agm}
Borja Balle and Yu-Xiang Wang.
\newblock Improving the gaussian mechanism for differential privacy: Analytical
  calibration and optimal denoising.
\newblock In \emph{Proceedings of the 35th International Conference on Machine
  Learning, {ICML}}, 2018.

\bibitem[Barthe and Olmedo(2013)]{barthe2013beyond}
Gilles Barthe and Federico Olmedo.
\newblock Beyond differential privacy: Composition theorems and relational
  logic for f-divergences between probabilistic programs.
\newblock In \emph{International Colloquium on Automata, Languages, and
  Programming}, pages 49--60. Springer, 2013.

\bibitem[Barthe et~al.(2012)Barthe, K{\"{o}}pf, Olmedo, and
  B{\'{e}}guelin]{BartheKOB12}
Gilles Barthe, Boris K{\"{o}}pf, Federico Olmedo, and Santiago~Zanella
  B{\'{e}}guelin.
\newblock Probabilistic relational reasoning for differential privacy.
\newblock In \emph{Symposium on Principles of Programming Languages ({POPL})},
  pages 97--110, 2012.

\bibitem[Barthe et~al.(2016)Barthe, Gaboardi, Gr{\'{e}}goire, Hsu, and
  Strub]{BartheGGHS16}
Gilles Barthe, Marco Gaboardi, Benjamin Gr{\'{e}}goire, Justin Hsu, and
  Pierre{-}Yves Strub.
\newblock Proving differential privacy via probabilistic couplings.
\newblock In \emph{Symposium on Logic in Computer Science ({LICS})}, pages
  749--758, 2016.

\bibitem[Bassily et~al.(2014)Bassily, Smith, and Thakurta]{bassily2014private}
Raef Bassily, Adam Smith, and Abhradeep Thakurta.
\newblock Private empirical risk minimization: Efficient algorithms and tight
  error bounds.
\newblock In \emph{Foundations of Computer Science (FOCS), 2014 IEEE 55th
  Annual Symposium on}, pages 464--473. IEEE, 2014.

\bibitem[Beimel et~al.(2010)Beimel, Kasiviswanathan, and
  Nissim]{beimel2010bounds}
Amos Beimel, Shiva~Prasad Kasiviswanathan, and Kobbi Nissim.
\newblock Bounds on the sample complexity for private learning and private data
  release.
\newblock In \emph{Theory of Cryptography Conference}, pages 437--454.
  Springer, 2010.

\bibitem[Beimel et~al.(2013)Beimel, Nissim, and
  Stemmer]{beimel2013characterizing}
Amos Beimel, Kobbi Nissim, and Uri Stemmer.
\newblock Characterizing the sample complexity of private learners.
\newblock In \emph{Proceedings of the 4th conference on Innovations in
  Theoretical Computer Science}, pages 97--110. ACM, 2013.

\bibitem[Beimel et~al.(2014)Beimel, Brenner, Kasiviswanathan, and
  Nissim]{beimel2014bounds}
Amos Beimel, Hai Brenner, Shiva~Prasad Kasiviswanathan, and Kobbi Nissim.
\newblock Bounds on the sample complexity for private learning and private data
  release.
\newblock \emph{Machine learning}, 94\penalty0 (3):\penalty0 401--437, 2014.

\bibitem[Bun and Steinke(2016)]{DBLP:conf/tcc/BunS16}
Mark Bun and Thomas Steinke.
\newblock Concentrated differential privacy: Simplifications, extensions, and
  lower bounds.
\newblock In \emph{Theory of Cryptography - 14th International Conference,
  {TCC} 2016-B, Beijing, China, October 31 - November 3, 2016, Proceedings,
  Part {I}}, pages 635--658, 2016.

\bibitem[Bun et~al.(2015)Bun, Nissim, Stemmer, and
  Vadhan]{bun2015differentially}
Mark Bun, Kobbi Nissim, Uri Stemmer, and Salil Vadhan.
\newblock Differentially private release and learning of threshold functions.
\newblock In \emph{Foundations of Computer Science (FOCS), 2015 IEEE 56th
  Annual Symposium on}, pages 634--649. IEEE, 2015.

\bibitem[Bun et~al.(2018)Bun, Dwork, Rothblum, and Steinke]{tcdp}
Mark Bun, Cynthia Dwork, Guy Rothblum, and Thomas Steinke.
\newblock Composable and versatile privacy via truncated cdp.
\newblock In \emph{Symposium on Theory of Computing, {STOC}}, 2018.

\bibitem[Chaudhuri and Mishra(2006)]{chaudhuri2006random}
Kamalika Chaudhuri and Nina Mishra.
\newblock When random sampling preserves privacy.
\newblock In \emph{Annual International Cryptology Conference}, pages 198--213.
  Springer, 2006.

\bibitem[Dwork and Roth(2014)]{dwork2014algorithmic}
Cynthia Dwork and Aaron Roth.
\newblock The algorithmic foundations of differential privacy.
\newblock \emph{Foundations and Trends in Theoretical Computer Science},
  9\penalty0 (3-4):\penalty0 211--407, 2014.

\bibitem[Dwork and Rothblum(2016)]{dwork2016concentrated}
Cynthia Dwork and Guy~N Rothblum.
\newblock Concentrated differential privacy.
\newblock \emph{arXiv preprint arXiv:1603.01887}, 2016.

\bibitem[Dwork et~al.(2010)Dwork, Rothblum, and Vadhan]{dwork2010boosting}
Cynthia Dwork, Guy~N Rothblum, and Salil Vadhan.
\newblock Boosting and differential privacy.
\newblock In \emph{Foundations of Computer Science (FOCS), 2010 51st Annual
  IEEE Symposium on}, pages 51--60. IEEE, 2010.

\bibitem[J{\"{a}}lk{\"{o}} et~al.(2017)J{\"{a}}lk{\"{o}}, Honkela, and
  Dikmen]{DBLP:conf/uai/JalkoHD17}
Joonas J{\"{a}}lk{\"{o}}, Antti Honkela, and Onur Dikmen.
\newblock Differentially private variational inference for non-conjugate
  models.
\newblock In \emph{Proceedings of the Thirty-Third Conference on Uncertainty in
  Artificial Intelligence, {UAI} 2017, Sydney, Australia, August 11-15, 2017},
  2017.

\bibitem[Kairouz et~al.(2017)Kairouz, Oh, and
  Viswanath]{kairouz2017composition}
Peter Kairouz, Sewoong Oh, and Pramod Viswanath.
\newblock The composition theorem for differential privacy.
\newblock \emph{IEEE Transactions on Information Theory}, 63\penalty0
  (6):\penalty0 4037--4049, 2017.

\bibitem[Kasiviswanathan et~al.(2011)Kasiviswanathan, Lee, Nissim,
  Raskhodnikova, and Smith]{kasiviswanathan2011can}
Shiva~Prasad Kasiviswanathan, Homin~K Lee, Kobbi Nissim, Sofya Raskhodnikova,
  and Adam Smith.
\newblock What can we learn privately?
\newblock \emph{SIAM Journal on Computing}, 40\penalty0 (3):\penalty0 793--826,
  2011.

\bibitem[Li et~al.(2012)Li, Qardaji, and Su]{li2012sampling}
Ninghui Li, Wahbeh Qardaji, and Dong Su.
\newblock On sampling, anonymization, and differential privacy or,
  k-anonymization meets differential privacy.
\newblock In \emph{Proceedings of the 7th ACM Symposium on Information,
  Computer and Communications Security}, pages 32--33. ACM, 2012.

\bibitem[Mironov(2017)]{DBLP:conf/csfw/Mironov17}
Ilya Mironov.
\newblock R{\'{e}}nyi differential privacy.
\newblock In \emph{30th {IEEE} Computer Security Foundations Symposium, {CSF}
  2017, Santa Barbara, CA, USA, August 21-25, 2017}, pages 263--275, 2017.

\bibitem[Murtagh and Vadhan(2016)]{murtagh2016complexity}
Jack Murtagh and Salil Vadhan.
\newblock The complexity of computing the optimal composition of differential
  privacy.
\newblock In \emph{Theory of Cryptography Conference}, pages 157--175.
  Springer, 2016.

\bibitem[{\"O}sterreicher(2002)]{osterreicher2002csiszar}
Ferdinand {\"O}sterreicher.
\newblock Csisz{\'a}r’s f-divergences-basic properties.
\newblock \emph{RGMIA Res. Rep. Coll}, 2002.

\bibitem[Park et~al.(2016{\natexlab{a}})Park, Foulds, Chaudhuri, and
  Welling]{DBLP:journals/corr/ParkFCW16a}
Mijung Park, James~R. Foulds, Kamalika Chaudhuri, and Max Welling.
\newblock Private topic modeling.
\newblock \emph{CoRR}, abs/1609.04120, 2016{\natexlab{a}}.

\bibitem[Park et~al.(2016{\natexlab{b}})Park, Foulds, Chaudhuri, and
  Welling]{DBLP:journals/corr/ParkFCW16b}
Mijung Park, James~R. Foulds, Kamalika Chaudhuri, and Max Welling.
\newblock Variational bayes in private settings {(VIPS)}.
\newblock \emph{CoRR}, abs/1611.00340, 2016{\natexlab{b}}.

\bibitem[Sason and Verd{\'u}(2016)]{sason2016f}
Igal Sason and Sergio Verd{\'u}.
\newblock $ f $-divergence inequalities.
\newblock \emph{IEEE Transactions on Information Theory}, 62\penalty0
  (11):\penalty0 5973--6006, 2016.

\bibitem[Ullman(2017)]{ullmanclass}
Jonathan Ullman.
\newblock Cs7880: Rigorous approaches to data privacy.
\newblock \url{http://www.ccs.neu.edu/home/jullman/PrivacyS17/HW1sol.pdf},
  2017.

\bibitem[Vadhan(2017)]{DBLP:books/sp/17/Vadhan17}
Salil~P. Vadhan.
\newblock The complexity of differential privacy.
\newblock In \emph{Tutorials on the Foundations of Cryptography.}, pages
  347--450. 2017.

\bibitem[Wang et~al.(2015)Wang, Fienberg, and Smola]{wang2015privacy}
Yu-Xiang Wang, Stephen Fienberg, and Alex Smola.
\newblock Privacy for free: Posterior sampling and stochastic gradient monte
  carlo.
\newblock In \emph{Proceedings of the 32nd International Conference on Machine
  Learning (ICML-15)}, pages 2493--2502, 2015.

\bibitem[Wang et~al.(2016)Wang, Lei, and Fienberg]{wang2016learning}
Yu-Xiang Wang, Jing Lei, and Stephen~E. Fienberg.
\newblock Learning with differential privacy: Stability, learnability and the
  sufficiency and necessity of erm principle.
\newblock \emph{Journal of Machine Learning Research}, 17\penalty0
  (183):\penalty0 1--40, 2016.

\bibitem[Wang et~al.(2018)Wang, Balle, and Kasiviswanathan]{wang2018subsampled}
Yu-Xiang Wang, Borja Balle, and Shiva Kasiviswanathan.
\newblock Subsampled r\'enyi differential privacy and analytical moments
  accountant.
\newblock \emph{ArXiv e-prints}, 2018.

\end{thebibliography}

\clearpage
\newpage
\appendix

\section{Proofs from Section~\ref{sec:tools}}

\begin{proof}[Proof of Theorem~\ref{thm:ajc}]
It suffices to check that for any $z \in \Z$,
\begin{align*}
[\mu(z) - \alpha' \mu'(z)]_+ =
\eta \left[\mu_1(z) - \alpha \left( (1-\beta) \mu_0(z) + \beta \mu_1'(z) \right)\right]_+ \enspace.
\end{align*}
Plugging this identity in the definition of $D_{\alpha'}$ we get the desired equality
\begin{align*}
D_{\alpha'}(\mu \| \mu')
=  \eta D_{\alpha} (\mu_1 \| (1-\beta) \mu_0 + \beta \mu_1') \enspace.
\end{align*}
\end{proof}

\begin{proof}[Proof of Theorem~\ref{thm:profileLap}]
Suppose $x \simeq_X x'$ and assume without loss of generality that $y = f(x) = 0$ and $y' = f(x) = \Delta > 0$. Plugging the density of the Laplace distribution in the definition of $\alpha$-divergence we get
\begin{align*}
D_{e^{\varepsilon}}(\Lap(b) \| \Delta + \Lap(b)) &= \frac{1}{2 b} \int_{\R} \left[ e^{- \frac{|z|}{b}} - e^{\varepsilon} e^{- \frac{|z - \Delta|}{b}} \right]_+ dz \enspace.
\end{align*}
Now we observe that the quantity inside the integral above is positive if and only if $|z - \Delta| - |z| \geq \varepsilon b$.
Since $||z + \Delta| - |z|| \leq \Delta$, we see that the divergence is zero for $\varepsilon > \Delta/b$.
On the other hand, for $\varepsilon \in [0, \Delta/b]$ we have $\{ z : |z - \Delta| - |z| \geq \varepsilon b \} = (-\infty, (\Delta - \varepsilon b)/2]$. Thus, we have
\begin{align*}
\frac{1}{2 b} \int_{\R} \left[ e^{- \frac{|z|}{b}} - e^{\varepsilon} e^{- \frac{|z - \Delta|}{b}} \right]_+ dz
&=
\frac{1}{2 b} \int_{-\infty}^{(\Delta - \varepsilon b)/2} e^{- \frac{|z|}{b}} dz  -
\frac{e^{\varepsilon}}{2 b} \int_{-\infty}^{(\Delta - \varepsilon b)/2} e^{- \frac{|z-\Delta|}{b}} dz \enspace.
\end{align*}
Now we can compute both integrals as probabilities under the Laplace distribution:
\begin{align*}
\frac{1}{2 b} \int_{-\infty}^{(\Delta - \varepsilon b)/2} e^{- \frac{|z|}{b}} dz
&=
\Pr\left[ \Lap(b) \leq \frac{\Delta - \varepsilon b}{2}\right] \\
&=
1 - \frac{1}{2} \exp\left(\frac{\varepsilon b - \Delta}{2 b}\right) \enspace, \\
\frac{e^{\varepsilon}}{2 b} \int_{-\infty}^{(\Delta - \varepsilon b)/2} e^{- \frac{|z-\Delta|}{b}} dz
&=
e^{\varepsilon} \Pr\left[ \Lap(b) \leq \frac{-\Delta - \varepsilon b}{2}\right] \\
&=
\frac{e^{\varepsilon}}{2} \exp\left(\frac{- \varepsilon b - \Delta}{2 b}\right) \enspace.
\end{align*}
Putting these two quantities together we finally get, for $\varepsilon \leq \Delta/b$:
\begin{align*}
D_{e^{\varepsilon}}(\Lap(b) \| \Delta + \Lap(b)) 
&=
1 - \exp\left(\frac{\varepsilon}{2} - \frac{\Delta}{2b}\right) \enspace.
\end{align*}
\end{proof}

\begin{proof}[Proof of Theorem~\ref{thm:PPtoMGF}]
Let $\varphi = \varphi_{\cM}^{x,x'}$, $L = L_{\cM}^{x,x'}$, $\tphi = \varphi_{\cM}^{x',x}$, and $\tilde{L} = L_{\cM}^{x',x}$.
Recall that for any non-negative random variable $\rvX$ one has $\Ex[\rvX] = \int_0^\infty \Pr[\rvX > t] dt$.
We use this to write the moment generating function of the corresponding privacy loss random variable for $s \geq 0$ as follows:
\begin{align*}
\varphi(s)
&=
\int_0^\infty \Pr[e^{s L} > t] dt \\
&=
\int_0^\infty \Pr\left[\frac{p(\rvX)}{q(\rvX)} > t^{1/s}\right] dt \enspace,
\end{align*}
where $\rvX \sim \mu$, and $p$ and $q$ represent the densities of $\mu$ and $\nu$ with respect to a fixed base measure.
Next we observe the probability inside the integral above can be decomposed in terms of a divergence and a second integral with respect to $q$:
\begin{align*}
\Pr\left[\frac{p(\rvX)}{q(\rvX)} > t^{1/s}\right]
&=
\Pr[p(\rvX) > t^{1/s} q(\rvX)] \\
&=
\Ex_{\mu}\left[ \one[p > t^{1/s} q] \right] \\
&=
\int \one[p(z) > t^{1/s} q(z)] p(z) dz \\
&=
\int \one[p(z) > t^{1/s} q(z)] (p(z) - t^{1/s} q(z)) dz
+ t^{1/s} \int \one[p(z) > t^{1/s} q(z)] q(z) dz \\
&=
\int [p(z) - t^{1/s} q(z)]_+ dz
+ t^{1/s} \int \one[p(z) > t^{1/s} q(z)] q(z) dz \\
&=
D_{t^{1/s}}(\mu \| \mu')
+ t^{1/s} \int \one[p(z) > t^{1/s}q(z)] q(z) dz
\enspace.
\end{align*}
Note the term $D_{t^{1/s}}(\mu \| \mu')$ above is not a divergence when $t^{1/s} < 1$. The integral term above can be re-written as a probability in terms of $\tL$ as follows:
\begin{align*}
\int \one[p(z) > t^{1/s} q(z)] q(z) dz
&=
\Pr[p(\rvY) > t^{1/s} q(\rvY)] \\
&=
\Pr\left[\frac{p(\rvY)}{q(\rvY)} > t^{1/s}\right] \\
&=
\Pr\left[e^{-\tL} > t^{1/s}\right] \enspace,
\end{align*}
where $\rvY \sim \mu'$.
Thus, integrating with respect to $t$ we get an expression for $\varphi(s)$ involving two terms that we will need to massage further:
\begin{align*}
\varphi(s)
&=
\int_{0}^{\infty} D_{t^{1/s}}(\mu \| \mu') dt
+
\int_{0}^{\infty} t^{1/s} \Pr\left[e^{-\tL} > t^{1/s}\right] dt \enspace.
\end{align*}
To compute the second integral in the RHS above we perform the change of variables $dt' = t^{1/s} dt$, which comes from taking $t' = t^{1 + 1/s}/(1 + 1/s)$, or, equivalently, $t = ((1+1/s) t')^{1/(1+1/s)}$. This allows us to introduce the moment generating function of $\tL$ as follows:
\begin{align*}
\int_{0}^{\infty} t^{1/s} \Pr\left[e^{-\tL} > t^{1/s}\right] dt
&=
\int_{0}^{\infty} \Pr\left[e^{-\tL} > ((1+1/s) t')^{1/(s+1)} \right] dt' \\
&=
\int_{0}^{\infty} \Pr\left[\frac{s}{s+1} e^{-(s+1) \tL} > t' \right] dt' \\
&=
\frac{s}{s+1} \Ex\left[e^{-(s+1) \tL}\right] \\
&=
\frac{s}{s+1} \tphi(-s-1) \enspace.
\end{align*}
Putting the derivations above together and substituting $\tphi(-s-1)$ for $\varphi(s)$ we see that
\begin{align*}
\varphi(s) = \frac{s}{s+1} \varphi(s) + \int_{0}^{\infty} D_{t^{1/s}}(\mu \| \mu') dt \enspace,
\end{align*}
or equivalently:
\begin{align*}
\varphi(s) = (s+1) \int_{0}^{\infty} D_{t^{1/s}}(\mu \| \mu') dt \enspace.
\end{align*}

Now we observe that some terms in the integral above cannot be bounded using an $\alpha$-divergence between $\mu$ and $\mu'$, e.g.\ for $t \in (0,1)$ the term $D_{t^{1/s}}(\mu \| \mu')$ is not a divergence.
Instead, using the definition of $D_{t^{1/s}}(\mu \| \mu')$ we can see that these terms are equal to by $1 - t^{1/s} + t^{1/s} D_{t^{-1/s}}(\mu' \| \mu)$, where the last term is now a divergence.
Thus, we split the integral in the expression for $\varphi(s)$ into two parts and obtain
\begin{align*}
\varphi(s) &= (s+1) \int_0^1 \left(1 - {t'}^{1/s} + {t'}^{1/s} D_{{t'}^{-1/s}}(\mu' \| \mu)\right) dt'
+
(s+1) \int_1^{\infty} D_{t^{1/s}}(\mu \| \mu') dt \\
&=
1 + (s+1) \int_0^1 {t'}^{1/s} D_{{t'}^{-1/s}}(\mu' \| \mu) dt' + (s+1) \int_1^{\infty} D_{t^{1/s}}(\mu \| \mu') dt \enspace.
\end{align*}

Finally, we can obtain the desired equation by performing a series of simple changes of variables $t ' = 1/t$, $\alpha = t^{1/s}$, and $\alpha = e^{\varepsilon}$:
\begin{align*}
\varphi(s) &=
1 + (s+1) \int_1^\infty t^{-2 - 1/s} D_{t^{1/s}}(\mu' \| \mu) dt + (s+1) \int_1^{\infty} D_{t^{1/s}}(\mu \| \mu') dt \\
&=
1 + s (s+1) \int_1^\infty \left( \alpha^{s-1} D_{\alpha}(\mu \| \mu') + \alpha^{-s-2}  D_{\alpha}(\mu' \| \mu) \right) d\alpha \\
&=
1 + s (s + 1) \int_{0}^{\infty} \left( e^{s \varepsilon} D_{e^{\varepsilon}}(\mu \| \mu') + e^{- (s+1) \varepsilon} D_{e^{\varepsilon}}(\mu' \| \mu) \right) d\varepsilon
\enspace.
\end{align*}
\end{proof}

\begin{proof}[Proof of Theorem~\ref{thm:dcc}]
The result follows from a few simple observations.
The first observation is that for any coupling $\pi \in C(\nu,\nu')$ and $y \in \supp(\nu')$ we have
\begin{align*}
\sum_{y'} \pi_{y,y'} \delta_{\cM,d(y,y')}(\varepsilon) &\geq \sum_{y'} \pi_{y,y'} \delta_{\cM,d(y,\supp(\nu'))}(\varepsilon)
\\
&= \sum_y \nu_y \delta_{\cM,d(y,\supp(\nu'))}(\varepsilon) \enspace,
\end{align*}
where the first inequality follows from $d(y,y') \geq d(y,\supp(\nu'))$ and the fact that $\delta_{\cM,k}(\varepsilon)$ is monotonically increasing with $k$. Thus the RHS of \eqref{eqn:nocoupling} is always a lower bound for the LHS.
Now let $\pi$ be a $d_Y$-compatible coupling. Since the support of $\pi$ only contains pairs $(y,y')$ such that $d(y,y') = d(y,\supp(\nu'))$, we see that
\begin{align*}
\sum_{y,y'} \pi_{y,y'} \delta_{\cM,d(y,y')}(\varepsilon) = \sum_{y,y'} \pi_{y,y'} \delta_{\cM,d(y,\supp(\nu'))}(\varepsilon) =
\sum_{y} \nu_y \delta_{\cM,d(y,\supp(\nu'))}(\varepsilon) \enspace.
\end{align*}
The result follows.
\end{proof}

\section{Proofs from Section~\ref{sec:examples}}

\begin{proof}[Proof of Theorem~\ref{thm:po}]
Using the tools from Section~\ref{sec:tools}, the analysis is quite straightforward.
Given $x,x' \in \btwo^{\cU}$ with $x \simeq_r x'$, we write $\omega = \Swo_\eta(x)$ and $\omega' = \Swo_\eta(x')$ and note that $\TV(\omega,\omega') = \eta$. Next we define $x_0 = x \cap x'$ and observe that either $x_0 = x$ or $x_0 = x'$ by the definition of $\simeq_r$. Let $\omega_0 = \Spo_\eta(x_0)$. Then the decompositions of $\omega$ and $\omega'$ induced by their maximal coupling have either $\omega_1 = \omega_0$ when $x = x_0$ or $\omega_1' = \omega_0$ when $x' = x_0$. Noting that applying advanced joined convexity in the former case leads to an additional cancellation we see that the maximum will be attained when $x' = x_0$. In this case the distribution $\omega_1$ is given by $\omega_1(y \cup \{v\}) = \omega_0(y)$. This observation yields an obvious $d_{\simeq_r}$-compatible coupling between $\omega_1$ and $\omega_0 = \omega_1'$: first sample $y'$ from $\omega_0$ and then build $y$ by adding $v$ to $y'$.
Since every pair of datasets generated by this coupling has distance one with respect to $d_{\simeq_r}$, Theorem~\ref{thm:dcc} yields the bound $\delta_{\cM'}(\varepsilon') \leq \eta \delta_{\cM}(\varepsilon)$.
\end{proof}

\begin{proof}[Proof of Theorem~\ref{thm:wo}]
The analysis proceeds along the lines of the previous proof.
First we note that for any $x, x' \in \btwo_n^{\cU}$ with $x \simeq_s x'$, the total variation distance between $\omega = \Swo_m(x)$ and $\omega' = \Swo_m(x')$ is given by $\eta = \TV(\omega,\omega') = m/n$.
Applying advanced joint convexity (Theorem~\ref{thm:ajc}) with the decompositions $\omega = (1-\eta) \omega_0 + \eta \omega_1$ and $\omega' = (1-\eta) \omega_0 + \eta \omega_1'$ given by the maximal coupling, the analysis of $D_{e^{\varepsilon'}}(\omega M \| \omega' M)$ reduces to bounding the divergences $D_{e^\varepsilon}(\omega_1 M \| \omega_0 M)$ and $D_{e^\varepsilon}(\omega_1 M \| \omega_1' M)$. In this case both quantities can be bounded by $\delta_{\cM}(\varepsilon)$ by constructing appropriate $d_{\simeq_s}$-compatible couplings and combining \eqref{eqn:anyPi} with Theorem~\ref{thm:dcc}.

We construct the couplings as follows. Suppose $v, v' \in \cU$ are the elements where $x$ and $x'$ differ: $x_v = x'_v + 1$ and $x'_{v'} = x_{v'} + 1$. Let $x_0 = x \cap x'$. Then we have $\omega_0 = \Swo_{m}(x_0)$. Furthermore, writing $\tilde{\omega}_1 = \Swo_{m-1}(x_0)$ we have $\omega_1(y) = \tilde{\omega}_1(y \cap x_0)$ and $\omega'_1(y) = \tilde{\omega}_1(y \cap x_0)$. Using these definitions we build a coupling $\pi_{1,1}$ between $\omega_1$ and $\omega_1'$ through the following generative process: sample $y_0$ from $\tilde{\omega}_1$ and then let $y = y_0 \cup \{v\}$ and $y' \cup \{v'\}$. Similarly, we build a coupling $\pi_{1,0}$ between $\omega_1$ and $\omega_0$ as follows: sample $y_0$ from $\tilde{\omega}_1$, sample $u$ uniformly from $x_0 \setminus y_0$, and then let $y = y_0 \cup \{v\}$ and $y' = y_0 \cup \{u\}$. It is obvious from these constructions that $\pi_{1,1}$ and $\pi_{0,1}$ are both $d_{\simeq_s}$-compatible.
Plugging these observations together, we get $\delta_{\cM'}(\varepsilon') \leq (m/n) \delta_{\cM}(\varepsilon)$.
\end{proof}

\begin{proof}[Proof of Theorem~\ref{thm:wr}]
To bound the privacy profile of the subsampled mechanism $\cM^{\Swr_m}$ on $\btwo_n^{\cU}$ with respect to $\simeq_s$ we start by noting that taking $x, x' \in \btwo_n^{\cU}$, $x \simeq_s x'$, the total variation distance between $\omega = \Swr_m(x)$ and $\omega' = \Swr_m(x')$ is given by $\eta = \TV(\omega,\omega') = 1 - (1 - 1/n)^m$.
To define appropriate mixture components for applying the advanced joint composition property we write $v$ and $v'$ for the elements where $x$ and $x'$ differ and $x_0 = x \cap x'$ for the common part between both datasets. Then we have $\omega_0 = \Swr_m(x_0)$. Furthermore, $\omega_1$ is the distribution obtained from sampling $\tilde{y}$ from $\tilde{\omega}_1 = \Swr_{m-1}(x)$ and building $y$ by adding one occurrence of $v$ to $\tilde{y}$.
Similarly, sampling $y'$ from $\omega_1'$ corresponds to adding $v'$ to a multiset sampled from $\Swr_{m-1}(x')$.

Now we construct appropriate distance-compatible couplings. First we let $\pi_{1,1} \in \cP(\N_m^{\cU} \times \N_m^{\cU})$ be the distribution given by sampling $y$ from $\omega_1$ as above and outputting the pair $(y,y')$ obtained by replacing each $v$ in $y$ by $v'$. It is immediate from this construction that $\pi_{1,1}$ is a $d_{\simeq_s}$-compatible coupling between $\omega_1$ and $\omega_1'$. Furthermore, using the notation from Theorem~\ref{thm:dcc} and the construction of the maximal coupling, we see that for $k \geq 1$:
\begin{align*}
\omega_1(Y_k) = \frac{\omega(Y_k) - (1-\eta) \omega_0(Y_k)}{\eta}
&=
\frac{\Pr_{y \sim \omega}[y_v = k]}{\eta} =
\frac{1}{\eta} \binom{m}{k} \left(\frac{1}{n}\right)^{k} \left(1 - \frac{1}{n}\right)^{m-k}
\enspace,
\end{align*}
where we used $\omega_0(Y_k) = 0$ since $\omega_0$ is supported on multisets that do not include $v$.
Therefore, the distributions $\mu_1 = \omega_1 M$ and $\mu_1' = \omega_1' M$ satisfy
\begin{align}\label{eqn:wrm1m1}
\eta D_{e^{\varepsilon}}(\mu_1 \| \mu_1') \leq \sum_{k=1}^m
\binom{m}{k} \left(\frac{1}{n}\right)^{k} \left(1 - \frac{1}{n}\right)^{m-k}
\delta_{\cM,k}(\varepsilon) \enspace.
\end{align}
On the other hand, we can build a $d_{\simeq_s}$-compatible coupling between $\omega_1$ and $\omega_0$ by first sampling $y$ from $\omega_1$ and then replacing each occurrence of $v$ by an element picked uniformly at random from $x_0$. Again, this shows that $D_{e^{\varepsilon}}(\mu_1 \| \mu_0)$ is upper bounded by the right hand side of \eqref{eqn:wrm1m1}.

Therefore, we conclude that
\begin{align*}
\delta_{\cM'}(\varepsilon') \leq
\sum_{k=1}^m
\binom{m}{k} \left(\frac{1}{n}\right)^{k} \left(1 - \frac{1}{n}\right)^{m-k}
\delta_{\cM,k}(\varepsilon) \enspace.
\end{align*}
\end{proof}

\begin{proof}[Proof of Theorem~\ref{thm:hybrid}]
Suppose $x \simeq_r x'$ with $|x| = n$ and $|x'| = n - 1$. This is the worst-case direction for the neighbouring relation like in the proof of Theorem~\ref{thm:po}. Let $\omega = \Swr_m(x)$ and $\omega = \Swr_m(x')$. We have $\eta = \TV(\omega,\omega') = 1 - (1- 1/n)^m$, and the factorization induced by the maximal coupling has $\omega_0 = \omega_1' = \omega'$ and $\omega_1$ is given by first sampling $\tilde{y}$ from $\Swr_{m-1}(x)$ and then producing $y$ by adding to $\tilde{y}$ a copy of the element $v$ where $x$ and $x'$ differ. This definition of $\omega_1$ suggests the following coupling between $\omega_1$ and $\omega_0$: first sample $y$ from $\omega_1$, then produce $y'$ by replacing each copy of $v$ with a element from $x'$ sampled independently and uniformly. By construction we see that this coupling is $d_{\simeq_s}$-compatible, so we can apply Theorem~\ref{thm:dcc}. Using the same argument as in the proof of Theorem~\ref{thm:wr} we see that $\eta \omega_1(Y_k) = \binom{m}{k} (1/n)^{k} (1 - 1/n)^{m-k}$. Thus, we finally get
\begin{align*}
D_{e^{\varepsilon'}}(\cM^{\Swr_m}(x) \| \cM^{\Swr_m}(x')) &= \eta D_{e^\varepsilon}(\omega_1 M  \| \omega_0 M) \\
&\leq \eta \sum_{k=1}^m \omega_1(Y_k) \delta_{\cM,k}(\varepsilon) \\
&= \sum_{k=1}^m \binom{m}{k} \left(\frac{1}{n}\right)^{k} \left(1 - \frac{1}{n}\right)^{m-k} \delta_{\cM,k}(\varepsilon) \enspace.
\end{align*}
\end{proof}

\begin{theorem}\label{thm:wrong}
Let $\cM : \btwo^{\cU} \to \cP(Z)$ be a mechanism with privacy profile $\delta_{\cM}$ with respect to $\simeq_s$. Then the privacy profile with respect of $\simeq_s$ of the subsampled mechanism $\cM' = \cM^{\Spo_\gamma} : \btwo_n^{\cU} \to \cP(Z)$ on datasets of size $n$ satisfies the following:
\begin{align*}
\delta_{\cM'}(\varepsilon') \leq
\gamma \beta \delta_{\cM}(\varepsilon) +
\gamma (1 - \beta) \left( \sum_{k=1}^{n-1} \tilde{\gamma}_k \delta_{\cM}(\varepsilon_k) 
+ \tilde{\gamma}_n \right) \enspace,
\end{align*}
where $\varepsilon' = \log(1 + \gamma (e^{\varepsilon} - 1))$, $\beta = e^{\varepsilon'}/e^{\varepsilon}$, $\varepsilon_k = \varepsilon + \log(\frac{\gamma}{1 - \gamma} (\frac{n}{k}-1))$,
and $\tilde{\gamma}_k = \binom{n-1}{k-1} \gamma^{k-1} (1-\gamma)^{n-k}$.
\end{theorem}
\begin{proof}[Proof of Theorem~\ref{thm:wrong}]
Suppose $x, x' \in \btwo_n^{\cU}$ are sets of size $n$ related by the substitution relation $\simeq_s$. Let $\omega = \Spo_\eta(x)$ and $\omega' = \Spo_\eta(x')$ and note that $\TV(\omega,\omega') = \eta$. Let $x_0 = x \cap x'$ and $v = x \setminus x_0$, $v' = x' \setminus x_0$. In this case the factorization induced by the maximal coupling is obtained by taking $\omega_0 = \Spo_\eta(x_0)$, $\omega_1(y \cup \{v\}) = \omega_0(y)$, and $\omega_1'(y \cup \{v'\}) = \omega_0(y)$. 
From this factorization we see it is easy to construct a coupling $\pi_{1,1}$ between $\omega_1$ and $\omega_1'$ that is $d_{\simeq_s}$-compatible. Therefore we have $D_{e^\varepsilon}(\omega_1 M \| \omega_1' M) \leq \delta_{\cM}(\varepsilon)$.

Since we have already identified that no $d_{\simeq_s}$-compatible coupling between $\omega_1$ and $\omega_0$ can exist, we shall further decompose these distributions ``by hand''. Let $\nu_k = \Swo_k(x_0)$ and note that $\nu_k$ corresponds to the distribution $\omega_0$ conditioned on $|y|=k$. Similarly, we define $\tilde{\nu}_k$ as the distribution corresponding to sampling $\tilde{y}$ from $\Swo_{k-1}(x_0)$ and outputting the set $y$ obtained by adding $v$ to $\tilde{y}$. Then $\tilde{\nu}_k$ equals the distribution of $\omega_1$ conditioned on $|y|=k$. Now we write $\gamma_k = \Pr_{y \sim \omega_0}[|y|=k] = \binom{n-1}{k} \gamma^k (1-\gamma)^{n-1-k}$ and $\tilde{\gamma}_k = \Pr_{y \sim \omega_1}[|y|=k] = \binom{n-1}{k-1} \gamma^{k-1} (1-\gamma)^{n-k}$. With these notations we can write the decompositions
$\omega_0 = \sum_{k=0}^{n-1} \gamma_k \nu_k$ and $\omega_1 = \sum_{k=1}^{n} \tilde{\gamma}_k \tilde{\nu}_k$.
Further, we observe that the construction of $\tilde{\nu}_k$ and $\nu_k$ shows there exist $d_{\simeq_s}$-compatible couplings between these pairs of distributions when $1 \leq k \leq n-1$, leading to $D_{e^{\varepsilon}}(\tilde{\nu}_k M \| \nu_k M) \leq \delta_{\cM}(\varepsilon)$.
To exploit this fact we first write
\begin{align*}
D_{e^{\varepsilon}}(\omega_1 M \| \omega_0 M)
&=
D_{e^{\varepsilon}}\left(
\sum_{k=1}^{n-1} \tilde{\gamma}_k \tilde{\nu}_k M + \tilde{\gamma}_n \tilde{\nu}_n M \middle\|
\gamma_0 \nu_0 M + \sum_{k=1}^{n-1} {\gamma}_k {\nu}_k M \right) \enspace.
\end{align*}
Now we use that $\alpha$-divergences can be applied to arbitrary non-negative measures, which are not necessarily probability measures, using the same definition we have used so far. Under this relaxation, given non-negative measures $\nu_i, \nu_i'$, $i = 1,2$, on a measure space $Z$ we have $D_{\alpha}(\nu_1 + \nu_2 \| \nu_1' + \nu_2') \leq D_{\alpha}(\nu_1 \| \nu_1') +  D_{\alpha}(\nu_2 \| \nu_2')$, $D_{\alpha}(a \nu_1 \| b \nu_2) = a D_{\alpha b / a}(\nu_1 \| \nu_2)$ for $a \geq 0$ and $b > 0$, and $D_{\alpha}(\nu_1 \| 0 ) = \nu_1(Z)$. Using these properties on the decomposition above we see that
\begin{align*}
D_{e^{\varepsilon}}(\omega_1 M \| \omega_0 M)
&\leq
\sum_{k=1}^{n-1} \tilde{\gamma}_k D_{e^{\varepsilon_k}}(\tilde{\nu}_k M \| {\nu}_k M) 
+ \tilde{\gamma}_n \\
&\leq
\sum_{k=1}^{n-1} \tilde{\gamma}_k \delta_{\cM}(\varepsilon_k) 
+ \tilde{\gamma}_n
\enspace,
\end{align*}
where $e^{\varepsilon_k} = (\gamma_k / \tilde{\gamma}_k) e^{\varepsilon} = (\gamma / (1-\gamma)) (n/k -1) e^{\varepsilon}$.
\end{proof}

\section{Proofs from Section~\ref{sec:lb}}

\begin{proof}[Proof of Lemma~\ref{lem:naturalS}]
We start by observing that for any $x \in X$ the distribution $\mu = \cM_{v,p}^{\cS}(x)$ must be a mixture $\mu = (1-\theta) \nu_0 + \theta \nu_1$ for some $\theta \in [0,1]$. This follows from the fact that there are only two possibilities $\nu_0$ and $\nu_1$ for $\cM_{v,p}(y)$ depending on whether $v \notin y$ or $v \in y$. Similarly, taking $x \simeq_X x'$ we get $\mu' = \cM_{v,p}^{\cS}(x')$ with $\mu' = (1-\theta') \nu_0 + \theta' \nu_1$ for some $\theta' \in [0,1]$. Assuming (without loss of generality) $\theta \geq \theta'$, we use the advanced joint convexity property of $D_{\alpha}$ to get
\begin{align*}
D_{e^{\varepsilon'}}(\mu \| \mu') &=
\theta D_{e^{\varepsilon}}(\nu_1 \| (1 - \theta' /\theta) \nu_0 + (\theta'/\theta) \nu_1) \\
&\leq
\theta (1 - \theta' /\theta) D_{e^{\varepsilon}}(\nu_1 \| \nu_0) = (\theta - \theta') \psi_p(\varepsilon) \leq \theta \psi_p(\varepsilon)\enspace,
\end{align*}
where $\varepsilon' = \log(1 + \theta(e^{\varepsilon}-1))$ and $\beta = e^{\varepsilon'}/e^{\varepsilon}$, and the inequality follows from joint convexity. Now note the inequalities above are in fact equalities when $\theta' = 0$, which is equivalent to the fact $v \notin x'$ because $\cS$ is a natural subsampling mechanism. Thus, observing that the function $\theta \mapsto \theta \psi_p(\log(1 + (e^{\varepsilon'}-1)/\theta))$ is monotonically increasing, we get
\begin{align*}
\sup_{x \simeq_X x'} D_{e^{\varepsilon'}}(\cM_{v,p}^{\cS}(x) \| \cM_{v,p}^{\cS}(x')) &=
\sup_{x \simeq_X x', v \notin x'} \theta \psi_p(\log(1 + (e^{\varepsilon'}-1)/\theta)) \\
&= \eta \psi_p(\log(1 + (e^{\varepsilon'}-1)/\eta)) = \eta \psi_p(\varepsilon) \enspace.
\end{align*}
\end{proof}

\section{Plots of Privacy Profiles}\label{sec:plots}

\begin{figure}[h]
\begin{center}
\begin{subfigure}[b]{0.45\textwidth}
\centering
\includegraphics[width=.72\textwidth]{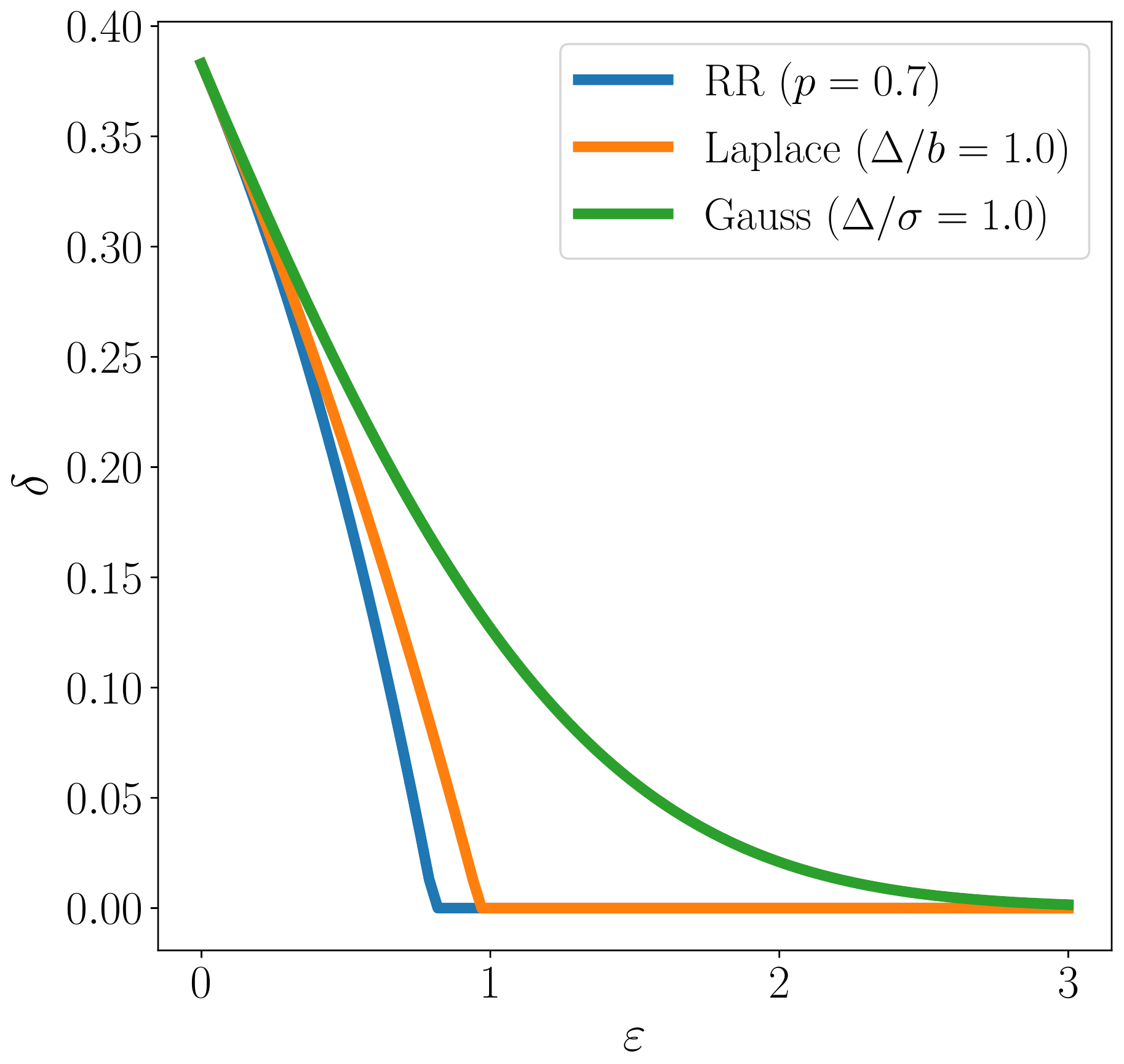}
\caption{\footnotesize Privacy profiles with mechanisms calibrated to provide the same $\delta$ at $\varepsilon = 0$. Profile expressions are given in Section~\ref{sec:lb} (RR), Theorem~\ref{thm:profileLap} (Laplace), and Theorem~\ref{thm:profileGau} (Gauss).}
\end{subfigure}\\[1em]
\begin{subfigure}[b]{0.45\textwidth}
\centering
\includegraphics[width=.72\textwidth]{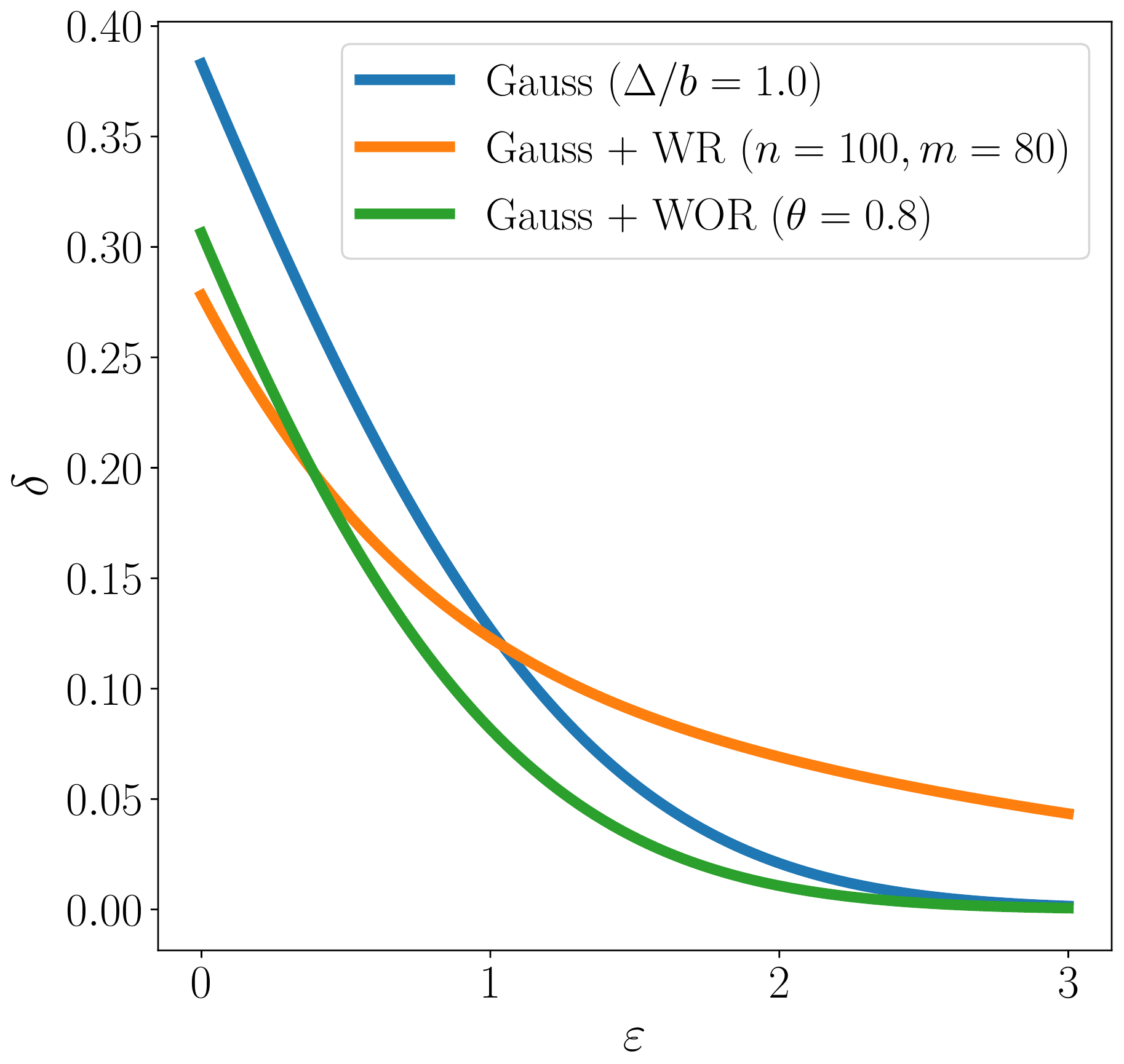}
\caption{\footnotesize Subsampled Gaussian mechanism. Comparison between sampling without replacement (Theorem~\ref{thm:wo}) and with replacement (Theorem~\ref{thm:wr}, with white-box group privacy), both with the same subsampled dataset sizes.}
\end{subfigure}
\hspace*{2em}
\begin{subfigure}[b]{0.45\textwidth}
\centering
\includegraphics[width=.72\textwidth]{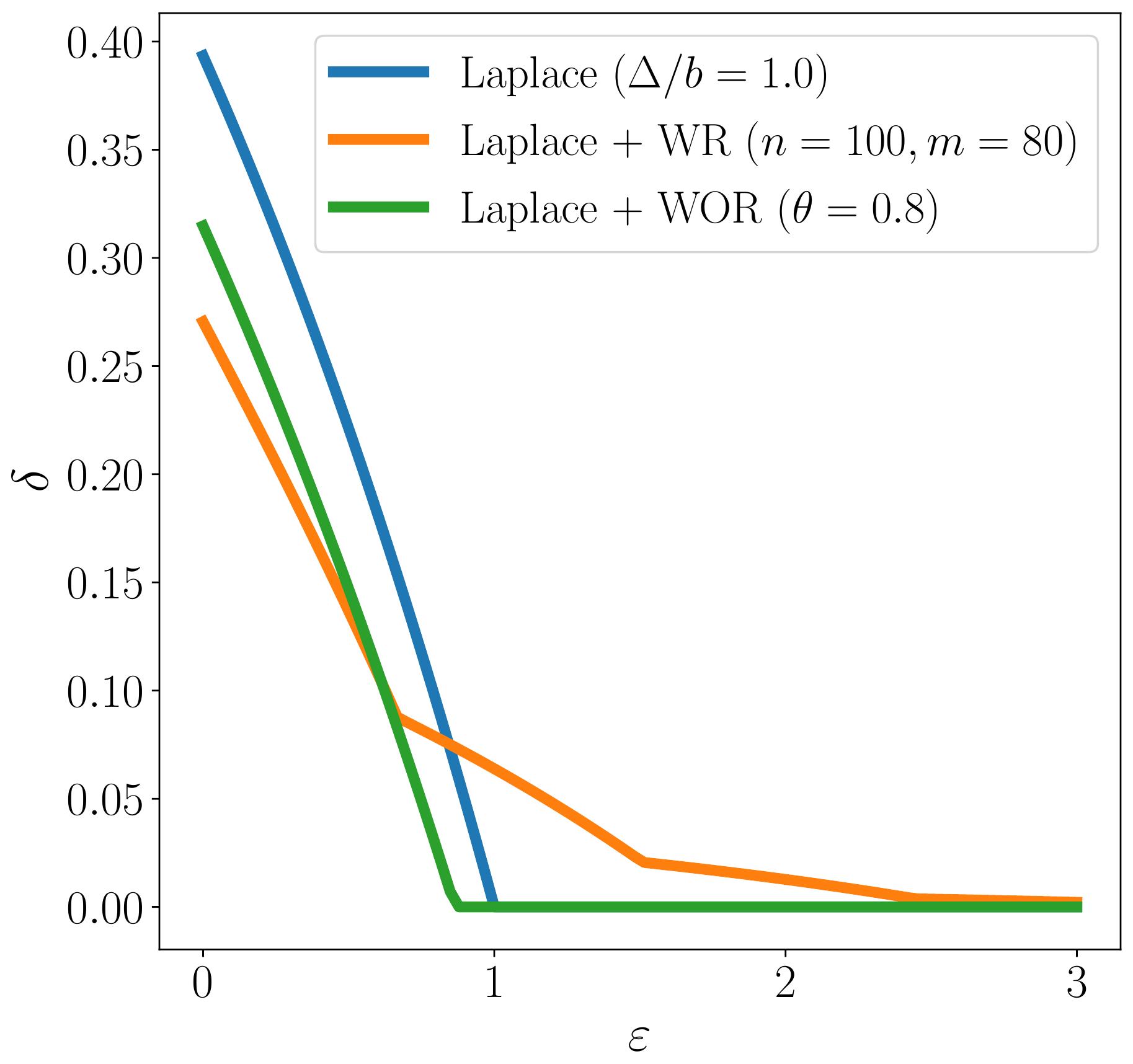}
\caption{\footnotesize Subsampled Laplace mechanism. Comparison between sampling without replacement (Theorem~\ref{thm:wo}) and with replacement (Theorem~\ref{thm:wr}, with white-box group privacy), both with the same subsampled dataset sizes.}
\end{subfigure}\\[1em]
\begin{subfigure}[b]{0.45\textwidth}
\centering
\includegraphics[width=.72\textwidth]{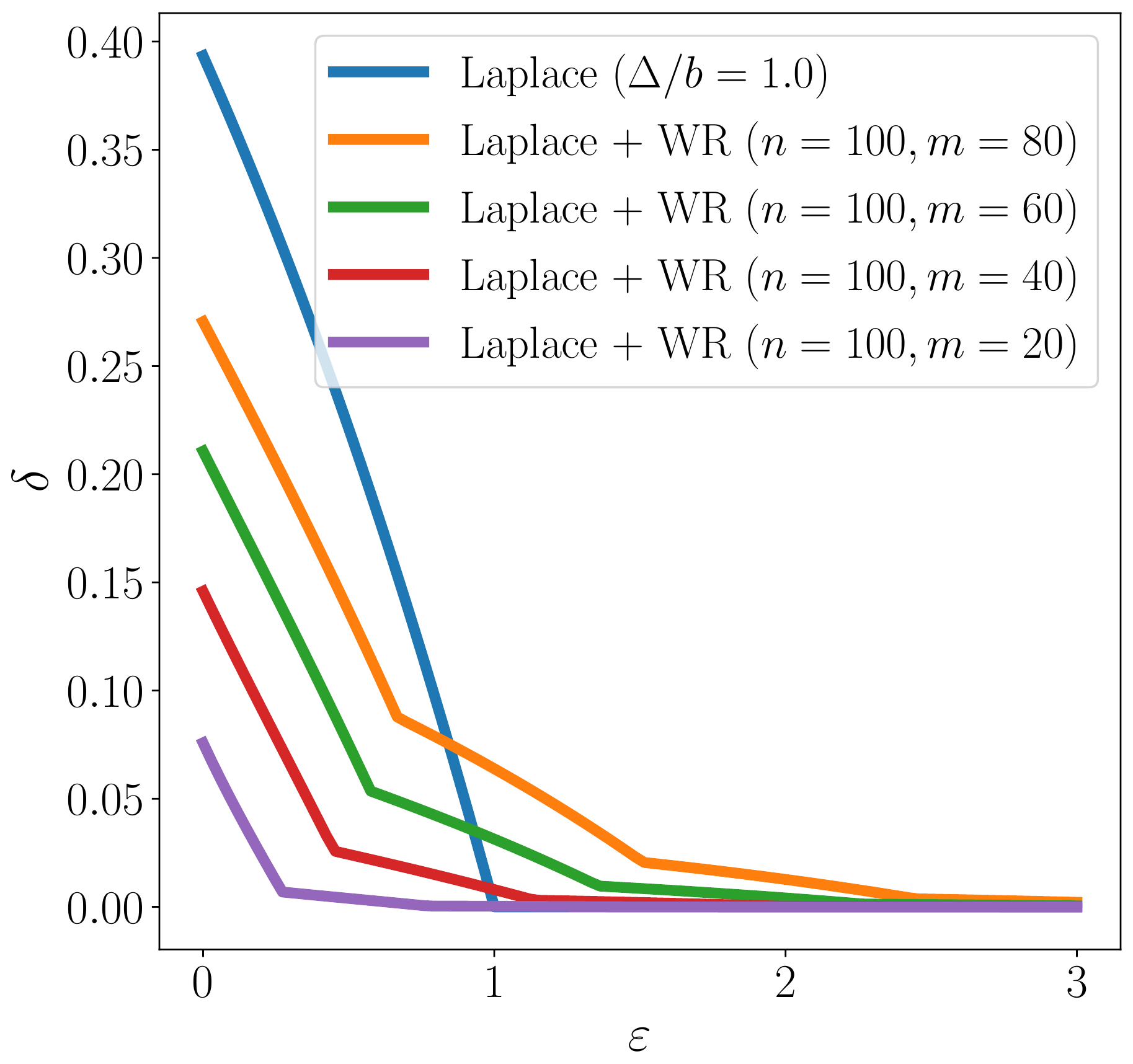}
\caption{\footnotesize Subsampled Laplace mechanism. Impact of group-privacy effect in sampling with replacement (white-box group privacy).}
\end{subfigure}
\hspace*{2em}
\begin{subfigure}[b]{0.45\textwidth}
\centering
\includegraphics[width=.72\textwidth]{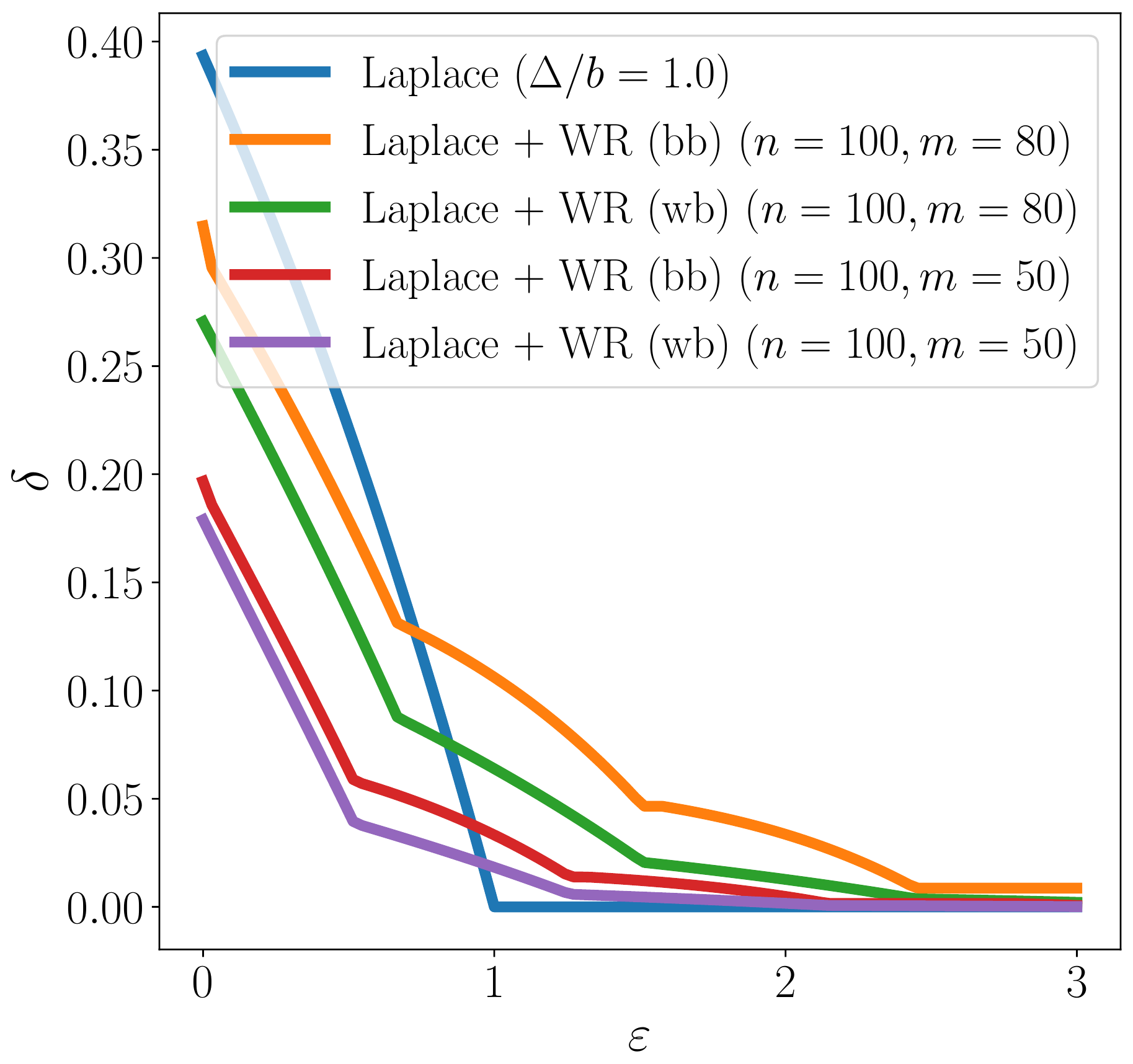}
\caption{\footnotesize Subsampled Laplace mechanism. Impact of white-box vs.\ black-box group-privacy in sampling with replacement.}
\end{subfigure}
\caption{Plots of privacy profiles. Results illustrate the notion of privacy profile and the different subsampling bounds derived in the paper.}
\label{fig:profiles}
\end{center}
\end{figure}

 \end{document}